%% file: arxiv.tex
\documentclass[]{utexas}
\usepackage[utf8]{inputenc} 
\usepackage[T1]{fontenc}    
\usepackage{hyperref}       
\usepackage{url}            
\usepackage{booktabs}       
\usepackage{amsfonts}       
\usepackage{nicefrac}       
\usepackage{microtype}      
\usepackage{xcolor}         

\usepackage{amsmath}
\usepackage{enumerate} 
\usepackage{algorithm}
\usepackage[noend]{algpseudocode}
\usepackage{amsfonts}
\usepackage{amsthm}
\usepackage{newtxtt}
\usepackage{cleveref}
\usepackage{diagbox} 
\usepackage{colortbl}
\usepackage{nicematrix}
\usepackage{subcaption}
\usepackage{multirow}

\usepackage{tcolorbox}
\usepackage[framemethod=tikz]{mdframed}
\usepackage{amssymb}
\usepackage{xspace}
\usepackage{wrapfig}
\usepackage{adjustbox}
\usepackage{tabularx}
\usepackage{mathtools}
\usepackage{xcolor}         

\usepackage{array}
\usepackage[table]{xcolor}
\usepackage{colortbl}
\usepackage[normalem]{ulem} 
\usepackage{makecell}  
\usepackage{graphicx}

\usepackage{xcolor}

\definecolor{OursGreen}{RGB}{232,245,233}      
\definecolor{ColGray}{RGB}{245,245,245}        
\definecolor{HeaderGray}{RGB}{250,250,250}

\newcommand{\best}[1]{\textbf{#1}}
\newcommand{\second}[1]{\uline{#1}}

\theoremstyle{plain}
\newtheorem{theorem}{Theorem}[section]
\newtheorem{proposition}[theorem]{Proposition}

\newtheorem{corollary}[theorem]{Corollary}
\theoremstyle{definition}

\theoremstyle{remark}

\usepackage[textsize=tiny]{todonotes}

\usepackage{tikz}

\title{TABES: Trajectory-Aware Backward-on-Entropy Steering for Masked Diffusion Models}

\author[1]{Shreshth Saini}
\author[1]{Avinab Saha\textdagger}
\author[2]{Balu Adsumilli}
\author[2]{Neil Birkbeck}
\author[2]{Yilin Wang}
\author[1]{Alan C. Bovik}

\affiliation[1]{The University of Texas at Austin}
\affiliation[2]{Google}

\footnotetext{\textdagger Work done while at UT Austin, now at Google Research.}

\input{sec/abstract}

\date{\today}
\correspondence{Shreshth Saini at \email{saini.2@utexas.edu}}

\begin{document}
\input{math_commands}

\maketitle




\input{sec/introduction}

\input{sec/background}

\input{sec/method}

\input{sec/experiments}

\input{sec/conclusion}


\clearpage
\newpage
\bibliographystyle{assets/plainnat}
\bibliography{refs.bib}


\input{sec/appendix}

\end{document}

%% file: sec/abstract.tex
\abstract{
Masked Diffusion Models (MDMs) have emerged as a promising non-autoregressive paradigm for generative tasks, offering parallel decoding and bidirectional context utilization. However, current sampling methods rely on simple confidence-based heuristics that ignore the long-term impact of local decisions, leading to trajectory lock-in where early hallucinations cascade into global incoherence. While search-based methods mitigate this, they incur prohibitive computational costs ($O(K)$ forward passes per step). In this work, we propose Backward-on-Entropy (BoE) Steering, a gradient-guided inference framework that approximates infinite-horizon lookahead via a single backward pass. We formally derive the Token Influence Score (TIS) from a first-order expansion of the trajectory cost functional, proving that the gradient of future entropy with respect to input embeddings serves as an optimal control signal for minimizing uncertainty. To ensure scalability, we introduce \texttt{ActiveQueryAttention}, a sparse adjoint primitive that exploits the structure of the masking objective to reduce backward pass complexity. BoE achieves a superior Pareto frontier for inference-time scaling compared to existing unmasking methods, demonstrating that gradient-guided steering offers a mathematically principled and efficient path to robust non-autoregressive generation. We will release the code.
}

%% file: math_commands.tex

\newcommand*{\vertbar}{\rule[-0.25ex]{0.5pt}{1.5ex}}
\newcommand*{\horzbar}{\rule[.5ex]{2.5ex}{0.5pt}}
\newcommand{\dd}{\mathrm{d}}
\newcommand{\tkernel}{p}
\newcommand{\action}[2]{\left \langle #1, #2\right \rangle }
\newcommand{\bell}{\mathrm{b}}
\newcommand{\norm}[1]
{\left\Vert#1\right\Vert}
\newcommand{\Norm}[1]{\lvert \! \lvert \! \lvert #1 \rvert \! \rvert \! \rvert}
\newcommand{\abs}[1]{\left\vert#1\right\vert}
\newcommand{\babs}[1]{\Big \vert#1 \Big \vert}
\newcommand{\set}[1]{\left\{#1\right\}}
\newcommand{\parr}[1]{\left (#1\right )}
\newcommand{\brac}[1]{\left [#1\right ]}
\newcommand{\ip}[1]{\left \langle #1 \right \rangle }
\newcommand{\Real}{\mathbb R}
\newcommand{\Nat}{\mathbb N}
\newcommand{\Complex}{\mathbb C}
\newcommand{\eps}{\varepsilon}
\newcommand{\too}{\rightarrow}
\newcommand{\bbar}[1]{\overline{#1}}
\newcommand{\wt}[1]{\widetilde{#1}} 
\newcommand{\wh}[1]{\widehat{#1}} 
\newcommand{\diag}{\textrm{diag}} 
\newcommand{\dist}{d} 
\newcommand{\divv}{\mathrm{div}} 
\newcommand{\vol}{\mathrm{vol}} 
\newcommand{\snr}{\mathrm{snr}}
\newcommand{\logsnr}{\rho}
\newcommand{\trace}{\textrm{tr}} 
\def \bfi{\textbf{\footnotesize{i}}} 
\newcommand{\one}{\mathbf{1}}
\newcommand{\zero}{\mathbf{0}}
\newcommand{\vcc}[1]{\mathrm{vec}(#1)}
\newcommand{\mat}[1]{\bm{[} #1 \bm{]}}
\newcommand{\defe}{\coloneqq}

\definecolor{mygray}{gray}{0.95}
\newcommand{\CM}{\scriptscriptstyle \text{CM}}
\newcommand{\M}{\scriptscriptstyle \text{M}}
\newcommand{\CFM}{\scriptscriptstyle \text{CFM}}
\newcommand{\FM}{\scriptscriptstyle \text{FM}}
\newcommand{\RFM}{\scriptscriptstyle \text{RFM}}
\newcommand{\RCFM}{\scriptscriptstyle \text{RCFM}}
\newcommand{\DFM}{\scriptscriptstyle \text{DFM}}
\newcommand{\CDFM}{\scriptscriptstyle \text{CDFM}}
\newcommand{\GM}{\scriptscriptstyle \text{GM}}
\newcommand{\DSM}{\scriptscriptstyle \text{DSM}}
\newcommand{\CGM}{\scriptscriptstyle \text{CGM}}
\newcommand{\SM}{\scriptscriptstyle \text{SM}}
\newcommand{\NM}{\scriptscriptstyle \text{NM}}
\newcommand{\mask}{\texttt{m}} 
\newcommand{\ignore}{\texttt{i}} 

\def \etal{{et al}.}
\newcommand*{\eg}{{\it e.g.}\@\xspace}
\newcommand*{\ie}{{\it i.e.}\@\xspace}

\makeatletter
\newtheorem*{rep@theorem}{\rep@title}
\newcommand{\newreptheorem}[2]{%
\newenvironment{rep#1}[1]{%
 \def\rep@title{\textbf{#2} \ref{##1}}%
 \begin{rep@theorem}}%
 {\end{rep@theorem}}}
\makeatother


\newreptheorem{theorem}{Theorem}
\newreptheorem{proposition}{Proposition}
\newreptheorem{lemma}{Lemma}
\newreptheorem{corollary}{Corollary}


\newcommand{\figleft}{{\em (Left)}}
\newcommand{\figcenter}{{\em (Center)}}
\newcommand{\figright}{{\em (Right)}}
\newcommand{\figtop}{{\em (Top)}}
\newcommand{\figbottom}{{\em (Bottom)}}
\newcommand{\captiona}{{\em (a)}}
\newcommand{\captionb}{{\em (b)}}
\newcommand{\captionc}{{\em (c)}}
\newcommand{\captiond}{{\em (d)}}

\newcommand{\newterm}[1]{{\bf #1}}

\def\figref#1{figure~\ref{#1}}
\def\Figref#1{Figure~\ref{#1}}
\def\twofigref#1#2{figures \ref{#1} and \ref{#2}}
\def\quadfigref#1#2#3#4{figures \ref{#1}, \ref{#2}, \ref{#3} and \ref{#4}}
\def\secref#1{section~\ref{#1}}
\def\Secref#1{Section~\ref{#1}}
\def\twosecrefs#1#2{sections \ref{#1} and \ref{#2}}
\def\secrefs#1#2#3{sections \ref{#1}, \ref{#2} and \ref{#3}}
\def\eqref#1{equation~\ref{#1}}
\def\Eqref#1{Equation~\ref{#1}}
\def\plaineqref#1{\ref{#1}}
\def\chapref#1{chapter~\ref{#1}}
\def\Chapref#1{Chapter~\ref{#1}}
\def\rangechapref#1#2{chapters\ref{#1}--\ref{#2}}
\def\algref#1{algorithm~\ref{#1}}
\def\Algref#1{Algorithm~\ref{#1}}
\def\twoalgref#1#2{algorithms \ref{#1} and \ref{#2}}
\def\Twoalgref#1#2{Algorithms \ref{#1} and \ref{#2}}
\def\partref#1{part~\ref{#1}}
\def\Partref#1{Part~\ref{#1}}
\def\twopartref#1#2{parts \ref{#1} and \ref{#2}}

\def\ceil#1{\lceil #1 \rceil}
\def\floor#1{\lfloor #1 \rfloor}
\def\1{\bm{1}}
\newcommand{\train}{\mathcal{D}}
\newcommand{\valid}{\mathcal{D_{\mathrm{valid}}}}
\newcommand{\test}{\mathcal{D_{\mathrm{test}}}}

\def\eps{{\epsilon}}




\def\reta{{\textnormal{$\eta$}}}
\def\ra{{\textnormal{a}}}
\def\rb{{\textnormal{b}}}
\def\rc{{\textnormal{c}}}
\def\rd{{\textnormal{d}}}
\def\re{{\textnormal{e}}}
\def\rf{{\textnormal{f}}}
\def\rg{{\textnormal{g}}}
\def\rh{{\textnormal{h}}}
\def\ri{{\textnormal{i}}}
\def\rj{{\textnormal{j}}}
\def\rk{{\textnormal{k}}}
\def\rl{{\textnormal{l}}}
\def\rn{{\textnormal{n}}}
\def\ro{{\textnormal{o}}}
\def\rp{{\textnormal{p}}}
\def\rq{{\textnormal{q}}}
\def\rr{{\textnormal{r}}}
\def\rs{{\textnormal{s}}}
\def\rt{{\textnormal{t}}}
\def\ru{{\textnormal{u}}}
\def\rv{{\textnormal{v}}}
\def\rw{{\textnormal{w}}}
\def\rx{{\textnormal{x}}}
\def\ry{{\textnormal{y}}}
\def\rz{{\textnormal{z}}}

\def\rvepsilon{{\mathbf{\epsilon}}}
\def\rvtheta{{\mathbf{\theta}}}
\def\rva{{\mathbf{a}}}
\def\rvb{{\mathbf{b}}}
\def\rvc{{\mathbf{c}}}
\def\rvd{{\mathbf{d}}}
\def\rve{{\mathbf{e}}}
\def\rvf{{\mathbf{f}}}
\def\rvg{{\mathbf{g}}}
\def\rvh{{\mathbf{h}}}
\def\rvu{{\mathbf{i}}}
\def\rvj{{\mathbf{j}}}
\def\rvk{{\mathbf{k}}}
\def\rvl{{\mathbf{l}}}
\def\rvm{{\mathbf{m}}}
\def\rvn{{\mathbf{n}}}
\def\rvo{{\mathbf{o}}}
\def\rvp{{\mathbf{p}}}
\def\rvq{{\mathbf{q}}}
\def\rvr{{\mathbf{r}}}
\def\rvs{{\mathbf{s}}}
\def\rvt{{\mathbf{t}}}
\def\rvu{{\mathbf{u}}}
\def\rvv{{\mathbf{v}}}
\def\rvw{{\mathbf{w}}}
\def\rvx{{\mathbf{x}}}
\def\rvy{{\mathbf{y}}}
\def\rvz{{\mathbf{z}}}

\def\erva{{\textnormal{a}}}
\def\ervb{{\textnormal{b}}}
\def\ervc{{\textnormal{c}}}
\def\ervd{{\textnormal{d}}}
\def\erve{{\textnormal{e}}}
\def\ervf{{\textnormal{f}}}
\def\ervg{{\textnormal{g}}}
\def\ervh{{\textnormal{h}}}
\def\ervi{{\textnormal{i}}}
\def\ervj{{\textnormal{j}}}
\def\ervk{{\textnormal{k}}}
\def\ervl{{\textnormal{l}}}
\def\ervm{{\textnormal{m}}}
\def\ervn{{\textnormal{n}}}
\def\ervo{{\textnormal{o}}}
\def\ervp{{\textnormal{p}}}
\def\ervq{{\textnormal{q}}}
\def\ervr{{\textnormal{r}}}
\def\ervs{{\textnormal{s}}}
\def\ervt{{\textnormal{t}}}
\def\ervu{{\textnormal{u}}}
\def\ervv{{\textnormal{v}}}
\def\ervw{{\textnormal{w}}}
\def\ervx{{\textnormal{x}}}
\def\ervy{{\textnormal{y}}}
\def\ervz{{\textnormal{z}}}

\def\rmA{{\mathbf{A}}}
\def\rmB{{\mathbf{B}}}
\def\rmC{{\mathbf{C}}}
\def\rmD{{\mathbf{D}}}
\def\rmE{{\mathbf{E}}}
\def\rmF{{\mathbf{F}}}
\def\rmG{{\mathbf{G}}}
\def\rmH{{\mathbf{H}}}
\def\rmI{{\mathbf{I}}}
\def\rmJ{{\mathbf{J}}}
\def\rmK{{\mathbf{K}}}
\def\rmL{{\mathbf{L}}}
\def\rmM{{\mathbf{M}}}
\def\rmN{{\mathbf{N}}}
\def\rmO{{\mathbf{O}}}
\def\rmP{{\mathbf{P}}}
\def\rmQ{{\mathbf{Q}}}
\def\rmR{{\mathbf{R}}}
\def\rmS{{\mathbf{S}}}
\def\rmT{{\mathbf{T}}}
\def\rmU{{\mathbf{U}}}
\def\rmV{{\mathbf{V}}}
\def\rmW{{\mathbf{W}}}
\def\rmX{{\mathbf{X}}}
\def\rmY{{\mathbf{Y}}}
\def\rmZ{{\mathbf{Z}}}

\def\ermA{{\textnormal{A}}}
\def\ermB{{\textnormal{B}}}
\def\ermC{{\textnormal{C}}}
\def\ermD{{\textnormal{D}}}
\def\ermE{{\textnormal{E}}}
\def\ermF{{\textnormal{F}}}
\def\ermG{{\textnormal{G}}}
\def\ermH{{\textnormal{H}}}
\def\ermI{{\textnormal{I}}}
\def\ermJ{{\textnormal{J}}}
\def\ermK{{\textnormal{K}}}
\def\ermL{{\textnormal{L}}}
\def\ermM{{\textnormal{M}}}
\def\ermN{{\textnormal{N}}}
\def\ermO{{\textnormal{O}}}
\def\ermP{{\textnormal{P}}}
\def\ermQ{{\textnormal{Q}}}
\def\ermR{{\textnormal{R}}}
\def\ermS{{\textnormal{S}}}
\def\ermT{{\textnormal{T}}}
\def\ermU{{\textnormal{U}}}
\def\ermV{{\textnormal{V}}}
\def\ermW{{\textnormal{W}}}
\def\ermX{{\textnormal{X}}}
\def\ermY{{\textnormal{Y}}}
\def\ermZ{{\textnormal{Z}}}

\def\vzero{{\bm{0}}}
\def\vone{{\bm{1}}}
\def\vmu{{\bm{\mu}}}
\def\vtheta{{\bm{\theta}}}
\def\va{{\bm{a}}}
\def\vb{{\bm{b}}}
\def\vc{{\bm{c}}}
\def\vd{{\bm{d}}}
\def\ve{{\bm{e}}}
\def\vf{{\bm{f}}}
\def\vg{{\bm{g}}}
\def\vh{{\bm{h}}}
\def\vi{{\bm{i}}}
\def\vj{{\bm{j}}}
\def\vk{{\bm{k}}}
\def\vl{{\bm{l}}}
\def\vm{{\bm{m}}}
\def\vn{{\bm{n}}}
\def\vo{{\bm{o}}}
\def\vp{{\bm{p}}}
\def\vq{{\bm{q}}}
\def\vr{{\bm{r}}}
\def\vs{{\bm{s}}}
\def\vt{{\bm{t}}}
\def\vu{{\bm{u}}}
\def\vv{{\bm{v}}}
\def\vw{{\bm{w}}}
\def\vx{{\bm{x}}}
\def\vy{{\bm{y}}}
\def\vz{{\bm{z}}}
\def\valpha{{\bm{\alpha}}}

\def\evalpha{{\alpha}}
\def\evbeta{{\beta}}
\def\evepsilon{{\epsilon}}
\def\evlambda{{\lambda}}
\def\evomega{{\omega}}
\def\evmu{{\mu}}
\def\evpsi{{\psi}}
\def\evsigma{{\sigma}}
\def\evtheta{{\theta}}
\def\eva{{a}}
\def\evb{{b}}
\def\evc{{c}}
\def\evd{{d}}
\def\eve{{e}}
\def\evf{{f}}
\def\evg{{g}}
\def\evh{{h}}
\def\evi{{i}}
\def\evj{{j}}
\def\evk{{k}}
\def\evl{{l}}
\def\evm{{m}}
\def\evn{{n}}
\def\evo{{o}}
\def\evp{{p}}
\def\evq{{q}}
\def\evr{{r}}
\def\evs{{s}}
\def\evt{{t}}
\def\evu{{u}}
\def\evv{{v}}
\def\evw{{w}}
\def\evx{{x}}
\def\evy{{y}}
\def\evz{{z}}

\def\mA{{\bm{A}}}
\def\mB{{\bm{B}}}
\def\mC{{\bm{C}}}
\def\mD{{\bm{D}}}
\def\mE{{\bm{E}}}
\def\mF{{\bm{F}}}
\def\mG{{\bm{G}}}
\def\mH{{\bm{H}}}
\def\mI{{\bm{I}}}
\def\mJ{{\bm{J}}}
\def\mK{{\bm{K}}}
\def\mL{{\bm{L}}}
\def\mM{{\bm{M}}}
\def\mN{{\bm{N}}}
\def\mO{{\bm{O}}}
\def\mP{{\bm{P}}}
\def\mQ{{\bm{Q}}}
\def\mR{{\bm{R}}}
\def\mS{{\bm{S}}}
\def\mT{{\bm{T}}}
\def\mU{{\bm{U}}}
\def\mV{{\bm{V}}}
\def\mW{{\bm{W}}}
\def\mX{{\bm{X}}}
\def\mY{{\bm{Y}}}
\def\mZ{{\bm{Z}}}
\def\mBeta{{\bm{\beta}}}
\def\mPhi{{\bm{\Phi}}}
\def\mphi{\bm{\phi}}
\def\mPsi{{\bm{\Psi}}}
\def\mpsi{\bm{\psi}}
\def\mLambda{{\bm{\Lambda}}}
\def\mSigma{{\bm{\Sigma}}}

\newcommand{\tens}[1]{\bm{\mathsfit{#1}}}
\def\tA{{\tens{A}}}
\def\tB{{\tens{B}}}
\def\tC{{\tens{C}}}
\def\tD{{\tens{D}}}
\def\tE{{\tens{E}}}
\def\tF{{\tens{F}}}
\def\tG{{\tens{G}}}
\def\tH{{\tens{H}}}
\def\tI{{\tens{I}}}
\def\tJ{{\tens{J}}}
\def\tK{{\tens{K}}}
\def\tL{{\tens{L}}}
\def\tM{{\tens{M}}}
\def\tN{{\tens{N}}}
\def\tO{{\tens{O}}}
\def\tP{{\tens{P}}}
\def\tQ{{\tens{Q}}}
\def\tR{{\tens{R}}}
\def\tS{{\tens{S}}}
\def\tT{{\tens{T}}}
\def\tU{{\tens{U}}}
\def\tV{{\tens{V}}}
\def\tW{{\tens{W}}}
\def\tX{{\tens{X}}}
\def\tY{{\tens{Y}}}
\def\tZ{{\tens{Z}}}

\def\gA{{\mathcal{A}}}
\def\gB{{\mathcal{B}}}
\def\gC{{\mathcal{C}}}
\def\gD{{\mathcal{D}}}
\def\gE{{\mathcal{E}}}
\def\gF{{\mathcal{F}}}
\def\gG{{\mathcal{G}}}
\def\gH{{\mathcal{H}}}
\def\gI{{\mathcal{I}}}
\def\gJ{{\mathcal{J}}}
\def\gK{{\mathcal{K}}}
\def\gL{{\mathcal{L}}}
\def\gM{{\mathcal{M}}}
\def\gN{{\mathcal{N}}}
\def\gO{{\mathcal{O}}}
\def\gP{{\mathcal{P}}}
\def\gQ{{\mathcal{Q}}}
\def\gR{{\mathcal{R}}}
\def\gS{{\mathcal{S}}}
\def\gT{{\mathcal{T}}}
\def\gU{{\mathcal{U}}}
\def\gV{{\mathcal{V}}}
\def\gW{{\mathcal{W}}}
\def\gX{{\mathcal{X}}}
\def\gY{{\mathcal{Y}}}
\def\gZ{{\mathcal{Z}}}

\def\sA{{\mathbb{A}}}
\def\sB{{\mathbb{B}}}
\def\sC{{\mathbb{C}}}
\def\sD{{\mathbb{D}}}
\def\sF{{\mathbb{F}}}
\def\sG{{\mathbb{G}}}
\def\sH{{\mathbb{H}}}
\def\sI{{\mathbb{I}}}
\def\sJ{{\mathbb{J}}}
\def\sK{{\mathbb{K}}}
\def\sL{{\mathbb{L}}}
\def\sM{{\mathbb{M}}}
\def\sN{{\mathbb{N}}}
\def\sO{{\mathbb{O}}}
\def\sP{{\mathbb{P}}}
\def\sQ{{\mathbb{Q}}}
\def\sR{{\mathbb{R}}}
\def\sS{{\mathbb{S}}}
\def\sT{{\mathbb{T}}}
\def\sU{{\mathbb{U}}}
\def\sV{{\mathbb{V}}}
\def\sW{{\mathbb{W}}}
\def\sX{{\mathbb{X}}}
\def\sY{{\mathbb{Y}}}
\def\sZ{{\mathbb{Z}}}

\def\emLambda{{\Lambda}}
\def\emA{{A}}
\def\emB{{B}}
\def\emC{{C}}
\def\emD{{D}}
\def\emE{{E}}
\def\emF{{F}}
\def\emG{{G}}
\def\emH{{H}}
\def\emI{{I}}
\def\emJ{{J}}
\def\emK{{K}}
\def\emL{{L}}
\def\emM{{M}}
\def\emN{{N}}
\def\emO{{O}}
\def\emP{{P}}
\def\emQ{{Q}}
\def\emR{{R}}
\def\emS{{S}}
\def\emT{{T}}
\def\emU{{U}}
\def\emV{{V}}
\def\emW{{W}}
\def\emX{{X}}
\def\emY{{Y}}
\def\emZ{{Z}}
\def\emSigma{{\Sigma}}

\newcommand{\etens}[1]{\mathsfit{#1}}
\def\etLambda{{\etens{\Lambda}}}
\def\etA{{\etens{A}}}
\def\etB{{\etens{B}}}
\def\etC{{\etens{C}}}
\def\etD{{\etens{D}}}
\def\etE{{\etens{E}}}
\def\etF{{\etens{F}}}
\def\etG{{\etens{G}}}
\def\etH{{\etens{H}}}
\def\etI{{\etens{I}}}
\def\etJ{{\etens{J}}}
\def\etK{{\etens{K}}}
\def\etL{{\etens{L}}}
\def\etM{{\etens{M}}}
\def\etN{{\etens{N}}}
\def\etO{{\etens{O}}}
\def\etP{{\etens{P}}}
\def\etQ{{\etens{Q}}}
\def\etR{{\etens{R}}}
\def\etS{{\etens{S}}}
\def\etT{{\etens{T}}}
\def\etU{{\etens{U}}}
\def\etV{{\etens{V}}}
\def\etW{{\etens{W}}}
\def\etX{{\etens{X}}}
\def\etY{{\etens{Y}}}
\def\etZ{{\etens{Z}}}

\newcommand{\pdata}{p_{\rm{data}}}
\newcommand{\ptrain}{\hat{p}_{\rm{data}}}
\newcommand{\Ptrain}{\hat{P}_{\rm{data}}}
\newcommand{\pmodel}{p_{\rm{model}}}
\newcommand{\Pmodel}{P_{\rm{model}}}
\newcommand{\ptildemodel}{\tilde{p}_{\rm{model}}}
\newcommand{\pencode}{p_{\rm{encoder}}}
\newcommand{\pdecode}{p_{\rm{decoder}}}
\newcommand{\precons}{p_{\rm{reconstruct}}}

\newcommand{\laplace}{\mathrm{Laplace}} 

\newcommand{\E}{\mathbb{E}}
\newcommand{\Ls}{\mathcal{L}}
\newcommand{\R}{\mathbb{R}}
\newcommand{\emp}{\tilde{p}}
\newcommand{\lr}{\alpha}
\newcommand{\reg}{\lambda}
\newcommand{\rect}{\mathrm{rectifier}}
\newcommand{\softmax}{\mathrm{softmax}}
\newcommand{\sigmoid}{\sigma}
\newcommand{\softplus}{\zeta}
\newcommand{\KL}{D_{\mathrm{KL}}}
\newcommand{\Var}{\mathrm{Var}}
\newcommand{\standarderror}{\mathrm{SE}}
\newcommand{\Cov}{\mathrm{Cov}}
\newcommand{\normlzero}{L^0}
\newcommand{\normlone}{L^1}
\newcommand{\normltwo}{L^2}
\newcommand{\normlp}{L^p}
\newcommand{\normmax}{L^\infty}

\newcommand{\parents}{Pa} 

\let\ab\allowbreak

\newcolumntype{C}[1]{>{\Centering}m{#1}}
\renewcommand\tabularxcolumn[1]{C{#1}}
\newcolumntype{Z}[1]{>{\Left}m{#1}}
\renewcommand\tabularxcolumn[1]{Z{#1}}

%% file: sec/introduction.tex

\begin{figure*}[t]
    \centering
    \includegraphics[width=0.9\textwidth]{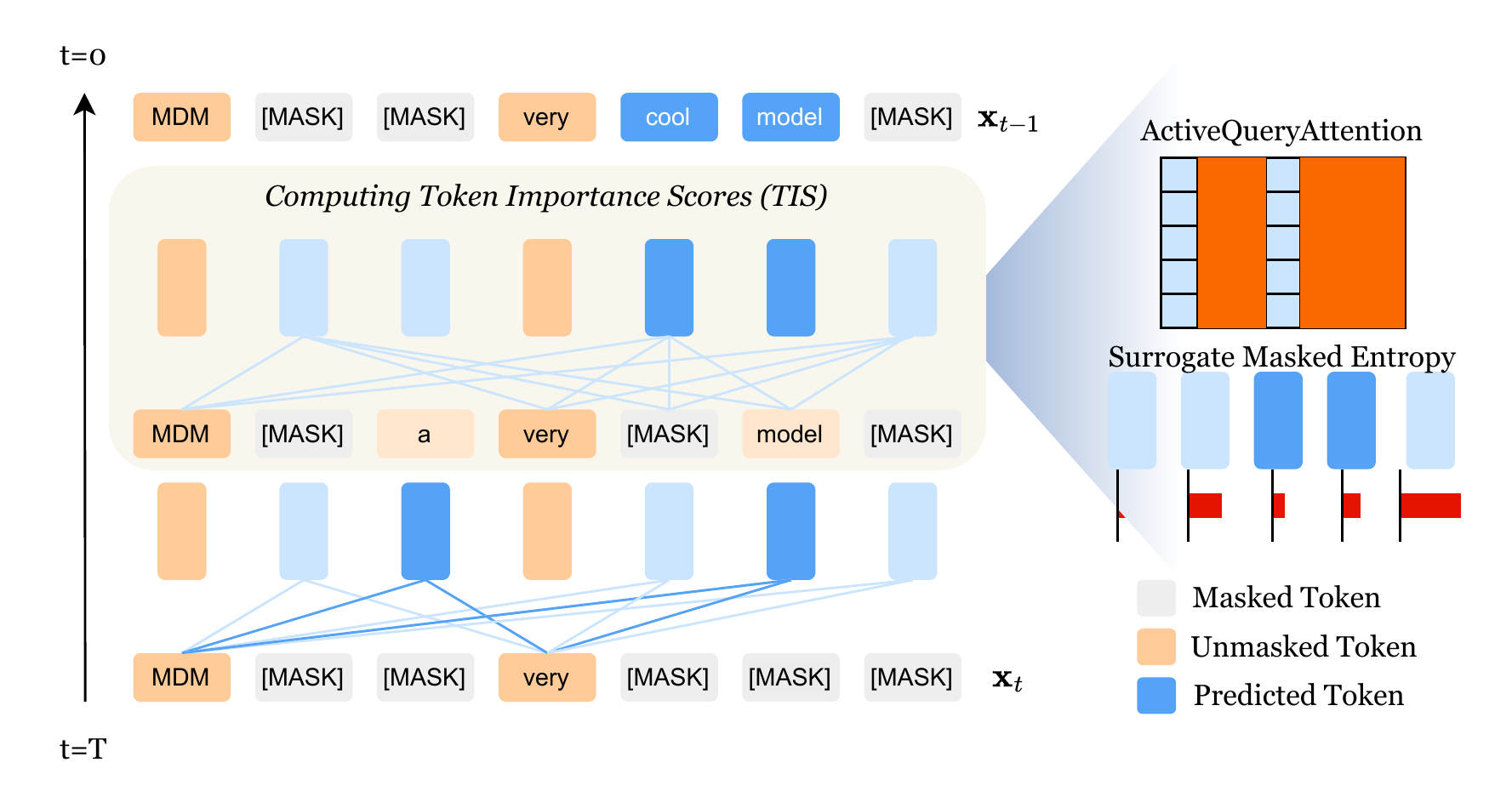}
    \caption{\textbf{Trajectory-aware sampling via Backward-on-Entropy steering.}
    Unlike heuristic schedules that unmask tokens with low current entropy, BoE selects positions that maximize expected reduction in future masked entropy, approximated with a single backward signal. \texttt{ActiveQueryAttention} restricts the backward computation to active positions, preserving inference efficiency.}
    \label{fig:boe_overview}
\end{figure*}

\section{Introduction}
\label{sec:intro}

The dominant paradigm in large-scale sequence generation is autoregressive (AR) modeling~\citep{achiam2023gpt, brown2020language, chowdhery2023palm, raffel2020exploring, team2024gemma, llama, llama2, yang2025qwen3}, which factorizes the joint distribution of a length-$L$ sequence $x=(x_1,\dots,x_L)$ as
\begin{equation}
p(x) \;=\; \prod_{t=1}^{L} p(x_t \mid x_{<t}).
\label{eq:ar_factorization}
\end{equation}
This causal factorization is a strong inductive bias for local consistency, and it supports mature inference accelerators such as KV caching~\citep{kv-survey, zeng2024context}. However, it also imposes an inherent sequential dependency during decoding, it prevents revision of earlier tokens after later context becomes available, suffers from linear latency scaling, and the reversal curse which limits reasoning capabilities~\citep{berglund2023reversal, golovneva2024reverse}. These limitations are especially consequential for long-horizon reasoning, where early commitments can be globally inconsistent yet locally plausible~\citep{havrilla2024glore}.

Masked Diffusion Models (MDMs)~\citep{mdlm, llada1.5, llada, ye2025dream} offer a fundamentally different route. Rather than committing to a fixed left-to-right order, MDMs define a noising process on partially observed sequences and learn to iteratively denoise from masked states toward complete sequences~\citep{llada,remdm,ebsampler}. This enables non-causal conditioning and, in principle, parallel token updates. As a result, MDMs are appealing for latency-sensitive deployment and for tasks where bidirectional context can repair earlier uncertainty.

Standard greedy strategies, such as the cosine schedule used in MaskGIT~\citep{maskgit}, unmask tokens based solely on local confidence $p(x_i | \mathbf{x}_t)$. We argue that this approach is heuristic as it prioritizes tokens that are locally stable but potentially structurally destabilizing~\citep{lookum, pathp2, gibbs-particle}. For instance, in a mathematical proof, choosing a locally plausible but incorrect intermediate step can render the entire subsequent proof incoherent, a phenomenon we term Trajectory Lock-in.

\textbf{Key observation.}
In MDM decoding, the choice of which token to unmask next is not a cosmetic scheduling detail. It is a control decision that determines what information becomes available to the model, thereby shaping the future posterior over the remaining masked tokens (see Fig.~\ref{fig:boe_overview}). Greedy schedules tend to reveal easy tokens that the model is already confident about, which yields small marginal information gain~\citep{pathp2,adlm, klass,jazbec2025learning}. Conversely, the tokens that disambiguate global structure, such as a latent variable in a multi-step derivation or a crucial entity binding in a narrative, may remain masked until late in decoding. This can force the sampler into globally inconsistent basins where subsequent iterations only refine a flawed scaffold. Existing solutions to this problem rely on expensive combinatorial search~\citep{lookum,pathp2,gibbs-particle}. For instance, Lookahead Unmasking (LookUM) \citep{lookum} simulate multiple future trajectories to verify local decisions. While effective, this increases inference compute by a factor of $K$, often negating the efficiency benefits of MDMs.

\textbf{Our approach.}
We formulate MDM sampling as a trajectory optimization problem under an entropy-regularized control objective~\citep{zhu2025mdns}. Concretely, we introduce Trajectory-Aware Backward-on-Entropy (BoE) Steering, a training-free sampler that uses a single backward (gradient) signal to estimate the expected reduction in future masked uncertainty induced by unmasking candidate tokens. BoE implements a one-step policy improvement viewpoint, rather than selecting tokens by their current entropy alone, it selects tokens that most reduce the next-step entropy of the remaining masked set under the model's denoising dynamics (see Fig.~\ref{fig:boe_overview}). This yields a principled lookahead mechanism that is compatible with existing MDMs. Our key insight is that the gradient of the entropy of the remaining masked tokens with respect to the current input embeddings provides a dense, first-order approximation of the value function for all candidate tokens simultaneously. 

\textbf{Systems challenge.}
A direct backward computation over all tokens is expensive and can negate inference benefits. To preserve the latency advantages of MDM decoding, we introduce \texttt{ActiveQueryAttention}, a sparsity-aware backward primitive tailored to the BoE objective. \texttt{ActiveQueryAttention} exploits the fact that the BoE gradient is only needed for a subset of active (masked) token positions, reducing the backward cost from quadratic attention-wide computation to a masked-active complexity that scales with the number of active positions. This makes gradient-informed steering practical for long sequences.

We summarize our contributions below:
\vspace{-6pt}
\begin{itemize}\setlength{\itemsep}{1pt}\setlength{\parskip}{0pt}\setlength{\parsep}{0pt}
    \item 
    We formalize token unmasking as a trajectory optimization. Based on this, we propose BoE steering, which selects unmasking actions by approximating the expected reduction in future masked entropy via a single backward signal.
    \item 
    We derive a practical per-token score (TIS) from a first-order expansion in embedding space that connects the backward signal to discrete token decisions.
    \item 
    We introduce a sparsity-aware \texttt{ActiveQueryAttention} mechanism that substantially reduces backward overhead, enabling compute-matched decoding comparisons.
    \item 
    We evaluate against ReMDM, entropy-bounded samplers, and strong schedule baselines on reasoning and long-context tasks, demonstrating improved accuracy--latency Pareto performance.
\end{itemize}

\begin{figure}[t]
    \centering
    \includegraphics[width=0.9\linewidth]{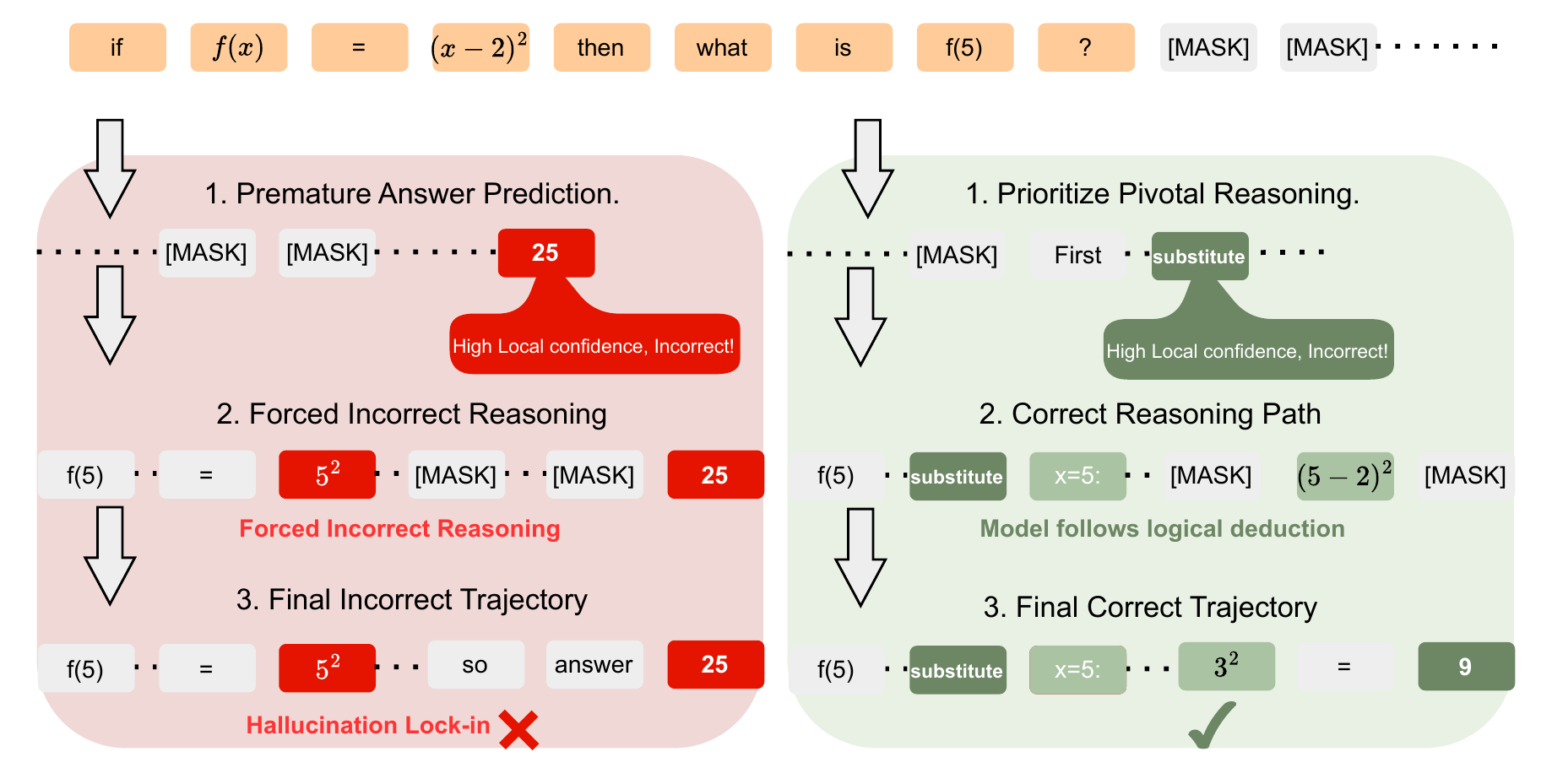}
    \caption{BoE (on right) prioritizes tokens that reduce future uncertainty unlike greedy schedule on left.
    On reasoning tasks, BoE tends to unmask globally load-bearing positions earlier, yielding faster reduction in total masked entropy and reduced incorrect trajectories. Best viewed zoomed-in.}
    \label{fig:boe_vs_greedy}
\end{figure}

%% file: sec/background.tex
\section{Background}
\label{sec:background}

\textbf{Notation.}
Let $\mathbf{x_0}=(x^1_0,\dots,x^L_0)\in\mathcal{V}^L$ be a length-$L$ token sequence over vocabulary $\mathcal{V}$.
Let $m\in\mathcal{V}$ denote a dedicated absorbing \texttt{[MASK]} token (we use $m$ and \texttt{[MASK]} interchangeably).
For any partially observed sequence $x_t\in\mathcal{V}^L$, define the masked index set
\begin{equation}
\mathcal{M}_t := \{ i\in\{1,\dots,L\} : x^i_t = m \}, \\
\qquad \overline{\mathcal{M}}_t := \{1,\dots,L\}\setminus \mathcal{M}_t.
\end{equation}
We write $x^{\overline{\mathcal{M}}_t}_t$ for observed (unmasked) tokens and $x^{\mathcal{M}_t}_t$ for masked tokens.

\textbf{Masked Diffusion Models.}
Masked Diffusion Models (MDMs) define a discrete noising process that progressively masks tokens, and a learned denoising process that reconstructs tokens from masked inputs \citep{mdlm, llada, lou2024discrete, austin2021structured}. Masked (absorbing-state) discrete diffusion defines a Markov chain $(x_t)_{t=0}^T$ on $\mathcal{V}^L$ that progressively replaces tokens with the absorbing symbol $m$. A widely used simplified masking process corrupts tokens independently with a monotonically decreasing schedule $\alpha_t\in[0,1]$ such that $\alpha_0=1$ and $\alpha_T=0$:
\begin{equation}
p_t(\mathbf{x_t} \mid \mathbf{x_0})
\;=\;
\prod_{i=1}^L p_t(x^i_t \mid x^i_0), \\
p_t(x^i_t \mid x^i_0)
=
\mathrm{Cat}\!\Big(x^i_t;\; \alpha_t\,\delta(x^i_0) + (1-\alpha_t)\,\delta(m)\Big),
\label{eq:forward_masking_common}
\end{equation}
where $\delta(\cdot)$ denotes a one-hot distribution over $\mathcal{V}$ and $\mathrm{Cat}(\cdot;\pi)$ is a categorical distribution with probabilities $\pi$.
A key structural property of \eqref{eq:forward_masking_common} is absorbing masking, i.e. once a token becomes $m$, it remains $m$ for all later forward times.

\begin{wrapfigure}[14]{r}{0.35\textwidth}
  \vspace{-15pt}
  \begin{center}
    \includegraphics[width=0.30\textwidth]{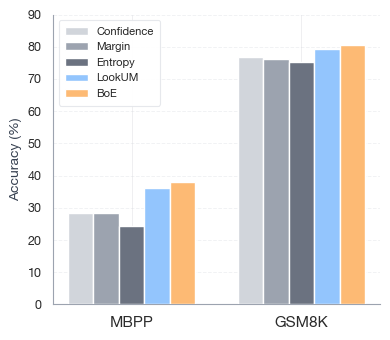}
  \end{center}
  \caption{Performance of greedy sampling with various unmasking method, and two trajectory-aware methods LookUM~\citep{lookum} and BoE(ours) from \texttt{LLaDA-8B}.}
  \vspace{1pt}
  \label{fig:boe-bars}
\end{wrapfigure}

\textbf{Decoding Process.}
The time-reversed posterior $q(x_{t-1}\mid x_t, x_0)$ factorizes across positions and takes a closed form.
For each position $i$, one obtains the two-case structure
\begin{equation}
q(x^i_{t-1} \mid x^i_t, x^i_0)
=
\begin{cases}
\mathrm{Cat}\!\big(x^i_{t-1}; \delta(x^i_t)\big), & x^i_t \neq m,\\[4pt]
\mathrm{Cat}\!\Big(x^i_{t-1}; \dfrac{1-\alpha_{t-1}}{1-\alpha_t}\,\delta(m) \\
+ \dfrac{\alpha_{t-1}-\alpha_t}{1-\alpha_t}\,\delta(x^i_0)\Big), & x^i_t = m,
\end{cases}
\label{eq:true_posterior_masked}
\end{equation}
valid for any discretization with $\alpha_{t-1}\ge \alpha_t$.
Equation \eqref{eq:true_posterior_masked} implies a locked-in behavior in the reverse direction, that is if a position is unmasked at time $t$ (that is, $x^i_t\neq m$), then $x^i_{t-1}=x^i_t$ deterministically. Equation \eqref{eq:true_posterior_masked} is parameterized to learn the reverse denoising process using a time-independent denoiser $D_\theta:\mathcal{V}^L\times\{1,\dots,T\}\to(\Delta^{|\mathcal{V}|-1})^L$ that predicts a distribution over clean tokens at $t=0$ given $x_t$. A standard induced reverse kernel replaces $\delta(x^i_0)$ in \eqref{eq:true_posterior_masked} with $D^i_\theta(x_t,t)$:
\begin{equation}
q_{t,\theta}(x^i_{t-1} \mid x^i_t, D^i_{\theta}(\mathbf{x_t}))
=
\begin{cases}
\mathrm{Cat}\!\big(x^i_{t-1}; \delta(x^i_t)\big), & x^i_t \neq m,\\[4pt]
\mathrm{Cat}\!\Big(
x^i_{t-1};
\dfrac{1-\alpha_{t-1}}{1-\alpha_t}\,\delta(m) \\
+
\dfrac{\alpha_{t-1}-\alpha_t}{1-\alpha_t}\,D^i_\theta(x_t,t)
\Big), & x^i_t = m.
\end{cases}
\label{eq:reverse_param_common}
\end{equation}
This parameterization makes explicit that reverse updates occur only at masked positions, and that already-unmasked positions remain fixed unless one modifies the inference procedure to allow revision~\citep{remdm}.

\textbf{Training objective.}
In the continuous-time limit and in widely used discrete approximations, the ELBO reduces to a weighted cross-entropy objective over masked positions.
A representative discrete-time surrogate is
\begin{equation}
\mathcal{L}(\theta)
=
\mathbb{E}_{x_0\sim p_{\text{data}}}
\;\mathbb{E}_{t\sim \mathrm{Unif}\{1,\dots,T\}}
\;\mathbb{E}_{x_t\sim q(\cdot\mid x_0)}
\Bigg[ \\
w_t \sum_{i: x^i_t=m} \mathrm{CE}\big(\delta(x^i_0),\; D^i_\theta(x_t,t)\big)
\Bigg],
\label{eq:training_objective_common}
\end{equation}
where $w_t$ is a known nonnegative weight determined by the schedule often proportional to $(\alpha_{t-1}-\alpha_t)/(1-\alpha_t)$. We omit derivations and refer to prior work for full ELBO details~\citep{elbo-discrete}.

\textbf{Unmasking Schedule.}
During inference, at each step, the sampler chooses (i) where to write by selecting an update set $\mathcal{U}_t\subseteq \mathcal{M}_t$ with budget $|\mathcal{U}_t|=b_t$, and (ii) what to write by sampling (or taking $\arg\max$) from the denoiser’s token distributions $p^i_\theta(\cdot\mid x_t,t):=D^i_\theta(x_t,t)$ for $i\in\mathcal{U}_t$, yielding the state update
\begin{equation}
x^i_{t-1}=
\begin{cases}
\hat{x}^i_{t-1}, & i\in \mathcal{U}_t,\\
x^i_t, & i\notin \mathcal{U}_t,
\end{cases}
\qquad
\mathcal{M}_{t-1}=\mathcal{M}_t\setminus \mathcal{U}_t,
\label{eq:state_update_common}
\end{equation}
with $\hat{x}^i_{t-1}\sim D^i_\theta(x_t,t)$.
Crucially, the scheduling decision $\mathcal{U}_t$ changes the conditioning context for all subsequent predictions, hence decoding is not merely token estimation under a fixed schedule but a coupled control problem in which the schedule governs information flow. Most practical policies instantiate $\mathcal{U}_t$ by ranking masked positions with a local uncertainty score $\sigma^i_t=\sigma(p^i_\theta(\cdot\mid x_t,t))$ and taking the top-$b_t$ indices, where $\sigma$ is typically confidence $p^{(i,1)}_t$~\citep{maskgit}, margin $p^{(i,1)}_t-p^{(i,2)}_t$~\citep{kim2025train}, or negative entropy $-H(p^i_\theta)$~\citep{koh2024plm}.
While effective, these greedy local-uncertainty schedules can prioritize tokens that are already easy to predict, whose revelation has low marginal influence on unresolved positions; in reasoning settings, this delays globally load-bearing pivots (e.g., an intermediate quantity, see Fig.~\ref{fig:boe_vs_greedy}), and the locked-in property of masked diffusion makes such delayed disambiguation difficult to correct without explicit remasking.
BoE addresses this limitation by prioritizing candidate updates according to their predicted reduction in future masked uncertainty. 

\begin{table*}[t]
\centering
\small
\setlength{\tabcolsep}{1pt}
\rowcolors{2}{HeaderGray}{white}
\begin{tabular}{lcccccc}
\toprule
\textbf{Method} &
\textbf{Training-free} &
\textbf{Lookahead} &
\textbf{Re-mask} &
\textbf{Backward signal} &
\textbf{Compute control} &
\textbf{Model-agnostic} \\
\midrule
Fixed schedule                         & \checkmark &     &       &       & \checkmark & \checkmark \\
Confidence/entropy-greedy              & \checkmark &     &       &       & \checkmark & \checkmark \\
Entropy-bounded (EB-style)             & \checkmark &     &       &       & \checkmark & \checkmark \\
Re-masking sampler (ReMDM-style)       & \checkmark &     & \checkmark  &       & \checkmark & \checkmark \\
Heuristic lookahead (LookUM-style)     & \checkmark & \checkmark&       &       & \textit{varies} & \checkmark \\
\midrule
\rowcolor{OursGreen}
\textbf{BoE Steering (ours)}           & \checkmark & \checkmark& \textit{optional} & \checkmark & \checkmark & \checkmark \\
\bottomrule
\end{tabular}
\caption{
\textbf{Sampling policy taxonomy (capabilities).}
BoE remains training-free and model-agnostic, adds principled short-horizon lookahead via a single backward signal, and preserves explicit compute control through sparse backward evaluation.
}
\label{tab:policy_compare}
\end{table*}

%% file: sec/method.tex
\section{Method: Backward-on-Entropy (BoE) Steering}
\label{sec:method}

BoE is an inference-time scheduling policy for masked diffusion decoding.
BoE works by unmasking the tokens that reduce the future masked token uncertainty. BoE is training-free and model-agnostic.

\subsection{One-Step Lookahead as Entropy-Regularized Control}
\label{sec:boe_control}
\textbf{Masked uncertainty.}
We measure uncertainty at time $t$ by the total predictive entropy over currently masked positions~\citep{ebsampler}:
\begin{equation}
\mathcal{H}_t(x_t)
\;:=\;
\sum_{i\in\mathcal{M}_t} H\!\left(p^i_\theta(\cdot\mid x_t,t)\right),
\label{eq:masked_entropy_t}
\end{equation}
\textbf{Lookahead objective.}
Greedy local-uncertainty schedules select $\mathcal{U}_t$ using only current-step scores~\citep{ebsampler,koh2024plm,kim2025train,elbo-discrete,klass}.
In contrast, BoE explicitly scores an action by its expected next-steps masked uncertainty after writing the selected tokens:
\begin{equation}
\mathcal{U}_t^\star
\;\in\;
\arg\min_{\mathcal{U}\subseteq\mathcal{M}_t,\;|\mathcal{U}|=b_t}
\;
\mathbb{E}_{\hat{x}_{t-1}^{\mathcal{U}}\sim \prod_{i\in\mathcal{U}}p^i_\theta(\cdot\mid x_t,t)} \\
\Big[
\mathcal{H}_{t-1}\!\big(x_{t-1}(\hat{x}_{t-1}^{\mathcal{U}})\big)
\Big]
\;+\;
\lambda_t\,\mathcal{R}_t(\mathcal{U}).
\label{eq:boe_objective_tminus}
\end{equation}
Here $x_{t-1}(\hat{x}_{t-1}^{\mathcal{U}})$ denotes the state obtained by writing $\hat{x}_{t-1}^i$ at $i\in\mathcal{U}$ and keeping other positions unchanged, and $\mathcal{R}_t$ is an anti-collapse regularizer (Section~\ref{sec:anticollapse}); $\lambda_t\ge 0$ is annealed.
Equation \eqref{eq:boe_objective_tminus} is combinatorial and involves an expectation over discrete writes.
BoE makes it tractable by (i) introducing a continuous relaxation for candidate writes, and (ii) using a single backward signal to approximate each position's marginal contribution to lookahead objective.

\subsection{Continuous Relaxation and a Differentiable One-Step Surrogate}
\label{sec:relaxation}


Let $E\in\mathbb{R}^{|\mathcal{V}|\times d}$ be the token embedding matrix and $e_m\in\mathbb{R}^d$ the embedding of the mask token $m$.
For a masked position $i\in\mathcal{M}_t$, define the denoiser distribution $\pi^i_t := p^i_\theta(\cdot\mid x_t,t)\in\Delta^{|\mathcal{V}|-1}$ and its soft write embedding
\begin{equation}
\tilde{e}^i_t \;:=\; E^\top \pi^i_t \in \mathbb{R}^d,
\qquad
\Delta e^i_t := \tilde{e}^i_t - e_m .
\label{eq:soft_embedding_t}
\end{equation}
For a candidate set $\mathcal{U}\subseteq\mathcal{M}_t$, define a relaxed next state $\tilde{x}_{t-1}(\mathcal{U})$ that is identical to $x_t$ except that positions in $\mathcal{U}$ are fed to the denoiser using their soft embeddings $\tilde{e}^i_t$ (after the embedding layer), while positions in $\mathcal{M}_t\setminus\mathcal{U}$ remain as $m$.
Using this relaxed state, we define the one-step surrogate masked entropy:
\begin{equation}
\widetilde{\mathcal{H}}_{t-1}(\mathcal{U})
\;:=\;
\sum_{j\in \mathcal{M}_t\setminus \mathcal{U}}
H\!\left(p^j_\theta(\cdot\mid \tilde{x}_{t-1}(\mathcal{U}),t-1)\right),
\label{eq:surrogate_next_entropy_tminus}
\end{equation}
which is differentiable with respect to the injected embeddings $\{\tilde{e}^i_t\}_{i\in\mathcal{U}}$.
Intuitively, $\widetilde{\mathcal{H}}_{t-1}(\mathcal{U})$ estimates how uncertain the remaining masked positions will be after revealing $\mathcal{U}$.

\subsection{Token Importance Score (TIS) from a Single Backward Signal}
\label{sec:tis}

BoE ranks candidates by their predicted reduction in future masked entropy.
For a candidate position $i\in\mathcal{M}_t$, consider continuously ``turning on'' its reveal from the mask embedding toward its soft write:
\begin{equation}
e^i(\alpha) \;=\; e_m + \alpha\,\Delta e^i_t,\qquad \alpha\in[0,1],
\label{eq:interp_tminus}
\end{equation}
with all other masked positions kept at $m$.
Let $\widetilde{\mathcal{H}}_{t-1}(\alpha;i)$ denote \eqref{eq:surrogate_next_entropy_tminus} under this single-position interpolation.
A first-order expansion around $\alpha=0$ gives
\begin{equation}
\widetilde{\mathcal{H}}_{t-1}(1;i)
\;\approx\;
\widetilde{\mathcal{H}}_{t-1}(0;i)
+
\left\langle
\nabla_{e_i}\widetilde{\mathcal{H}}_{t-1}(0;i),
\;\Delta e^i_t
\right\rangle.
\label{eq:taylor_tminus}
\end{equation}
Therefore the predicted entropy decrease from revealing $i$ is approximately the negative directional derivative in \eqref{eq:taylor_tminus}.
We define the Token Importance Score (TIS):
\begin{equation}
\mathrm{TIS}^i_t
\;:=\;
-\left\langle g^i_t,\;\Delta e^i_t\right\rangle,
\qquad
g^i_t := \nabla_{e_i}\widetilde{\mathcal{H}}_{t-1}(0;i).
\label{eq:tis_def_tminus}
\end{equation}
Here $g^i_t$ is a single-step backward signal measuring how sensitive next-step masked entropy is to revealing information at $i$.

\textbf{Confidence gating.}
To stabilize gradients on very high-entropy positions, we gate TIS with a monotone confidence factor
\begin{equation}
c^i_t \;=\; \mathrm{clip}\!\left(1 - \frac{H(p^i_\theta(\cdot\mid x_t,t))}{\log|\mathcal{V}|},\,0,\,1\right),
\\ 
\widehat{\mathrm{TIS}}^i_t \;=\; c^i_t\,\mathrm{TIS}^i_t .
\label{eq:confidence_gating_tminus}
\end{equation}
BoE selects $\mathcal{U}_t$ as the top-$b_t$ positions in a candidate set $\mathcal{C}_t\subseteq\mathcal{M}_t$ by $\widehat{\mathrm{TIS}}^i_t$, optionally incorporating the anti-collapse regularizer, Section ~\ref{sec:anticollapse}.


\subsection{Theoretical Guarantee for the First-Order Surrogate}
\label{sec:theory_main}

We formalize when TIS approximates true one-step entropy reduction in the relaxed objective.

\begin{theorem}[First-order entropy-reduction surrogate]
\label{thm:first_order_tminus}
Fix $t$ and $i\in\mathcal{M}_t$.
Assume $\widetilde{\mathcal{H}}_{t-1}(\alpha;i)$ is twice continuously differentiable on $[0,1]$ and
$\sup_{\alpha\in[0,1]}\big|\frac{d^2}{d\alpha^2}\widetilde{\mathcal{H}}_{t-1}(\alpha;i)\big|\le M$.
Define $\Delta\widetilde{\mathcal{H}}_{t-1}(i):=\widetilde{\mathcal{H}}_{t-1}(0;i)-\widetilde{\mathcal{H}}_{t-1}(1;i)$.
Then
\begin{equation}
\Delta\widetilde{\mathcal{H}}_{t-1}(i)
\;=\;
\mathrm{TIS}^i_t + \varepsilon^i_t,
\qquad
|\varepsilon^i_t|\le \frac{M}{2}.
\label{eq:error_bound_tminus}
\end{equation}
\end{theorem}

\begin{proof}
Provided in Appendix~\ref{sec:appendix_theory}.
\end{proof}

\begin{corollary}[Ordering stability under margin]
\label{cor:ordering}
Let $i,j\in\mathcal{M}_t$ satisfy the assumptions of Theorem~\ref{thm:first_order_tminus} with the same curvature bound $M$.
If $\mathrm{TIS}^i_t - \mathrm{TIS}^j_t > M$, then $\Delta\widetilde{\mathcal{H}}_{t-1}(i) > \Delta\widetilde{\mathcal{H}}_{t-1}(j)$; i.e., TIS ranks $i$ above $j$ in true relaxed one-step entropy reduction.
\end{corollary}

\subsection{Efficient Backward Scoring via \texttt{ActiveQueryAttention}}
\label{sec:aqa}

Computing $\{g^i_t\}_{i\in\mathcal{M}_t}$ naively requires backpropagating through attention for all $L$ queries, which is costly for long sequences.
BoE only needs gradients for a small active set.

\textbf{Active set and candidate prefilter.}
Let $\mathcal{C}_t\subseteq\mathcal{M}_t$ be a candidate set produced by a cheap prefilter (e.g., top-$r$ by confidence/margin/entropy).
Define the active query set $\mathcal{A}_t := \mathcal{C}_t$.
We compute forward attention normally to preserve logits, but restrict backward to queries in $\mathcal{A}_t$ by stopping gradients for inactive query outputs. Appendix~\ref{sec:appendix_code} provides detailed implementation guide. 

\textbf{Attention backward complexity.}
For a standard attention layer with hidden size $d$, the dominant score computation is $QK^\top$.
Standard backward scales as $O(L^2 d)$.
With active queries, backward scales as
\begin{equation}
O(|\mathcal{A}_t|\,L\,d),
\label{eq:aqa_complexity}
\end{equation}
while forward predictions remain unchanged.
This makes gradient-guided lookahead practical under compute-matched decoding.


\subsection{Anti-Collapse Regularizer}
\label{sec:anticollapse}

A common failure mode of aggressive steering is premature overcommitment, revealing very low-entropy tokens early can lock the incorrect trajectory early on. We discourage overly confident reveals early by penalizing tokens whose entropy is below a time-dependent floor. Choose a decreasing entropy floor schedule $h_t$ (high at early denoising $t\approx T$, low near completion $t\approx 1$), i.e.:
\begin{equation}
h_t = h_{\max}\cdot \frac{t}{T}.
\label{eq:entropy_floor_schedule}
\end{equation}
Define
\begin{equation}
\mathcal{R}_t(\mathcal{U})
\;:=\;
\sum_{i\in\mathcal{U}}
\big[h_t - H^i_t\big]_+^2,
\qquad
[a]_+ := \max(a,0).
\label{eq:anticollapse_tminus}
\end{equation}
In practice, we incorporate $\mathcal{R}_t$ by subtracting $\lambda_t [h_t - H^i_t]_+^2$ from each candidate score (equivalently, adding $\lambda_t\mathcal{R}_t$ to the objective in \eqref{eq:boe_objective_tminus}).

\subsection{BoE Steering Algorithm}
\label{sec:algorithm}

Algorithm~\ref{alg:boe} summarizes BoE.
Compared to greedy local-uncertainty schedules and entropy-bounded rules \citep{ebsampler}, BoE introduces a principled one-step lookahead signal.
Compared to remasking samplers \citep{remdm}, BoE can be used with or without remasking; our default is training-free scheduling without altering the underlying denoiser.
Compared to heuristic lookahead \citep{lookum}, BoE uses a single backward pass rather than multi-path forward search.

\begin{algorithm}[t]
\caption{Backward-on-Entropy (BoE) Steering for Masked Diffusion Decoding}
\label{alg:boe}
\begin{algorithmic}[1]
\Require Denoiser $D_\theta(\cdot,t)$, steps $T$, budgets $\{b_t\}_{t=1}^T$, candidate size $r$ (optional), anti-collapse weights $\{\lambda_t\}$
\State Initialize $x_T \leftarrow (m,\dots,m)$, $\mathcal{M}_T\leftarrow \{1,\dots,L\}$
\For{$t=T$ \textbf{down to} $1$}
    \State \textbf{Current-step prediction:} compute $\pi^i_t \leftarrow D^i_\theta(x_t,t)$ and entropies $H^i_t$ for $i\in\mathcal{M}_t$
    \State \textbf{Candidate prefilter:} choose $\mathcal{C}_t\subseteq\mathcal{M}_t$; else $\mathcal{C}_t=\mathcal{M}_t$
    \State \textbf{Build relaxed next state:} form soft embeddings $\tilde{e}^i_t=E^\top \pi^i_t$ for $i\in\mathcal{C}_t$ and define $\tilde{x}_{t-1}(\alpha;i)$ via \eqref{eq:interp_tminus}
    \State \textbf{One-step lookahead entropy:} compute $\widetilde{\mathcal{H}}_{t-1}$ in \eqref{eq:surrogate_next_entropy_tminus} with a forward pass at time $(t-1)$
    \State \textbf{Sparse backward (AQA):} set active queries $\mathcal{A}_t\leftarrow \mathcal{C}_t$ and compute gradients $g^i_t=\nabla_{e_i}\widetilde{\mathcal{H}}_{t-1}(0;i)$ for $i\in\mathcal{C}_t$
    \State \textbf{Scoring:} compute $\widehat{\mathrm{TIS}}^i_t$ via \eqref{eq:tis_def_tminus}--\eqref{eq:confidence_gating_tminus}, and subtract $\lambda_t [h_t-H^i_t]_+^2$
    \State \textbf{Select updates:} $\mathcal{U}_t \leftarrow \mathrm{TopK}\big(\widehat{\mathrm{TIS}}^i_t: i\in\mathcal{C}_t\big)$ with $|\mathcal{U}_t|=b_t$
    \State \textbf{Write tokens:} for $i\in\mathcal{U}_t$, set $\hat{x}^i_{t-1}\sim \pi^i_t$ (or $\arg\max$); update $x_{t-1}$ by \eqref{eq:state_update_common}
\EndFor
\State \Return $x_0$
\end{algorithmic}
\end{algorithm}

\textbf{Discussion.}
BoE replaces purely current-step local scoring with a lookahead proxy that targets future entropy reduction.
This promotes revealing globally load-bearing pivots earlier, which is particularly important under the locked-in property of masked diffusion (Section~\ref{sec:related}).
\texttt{ActiveQueryAttention} makes the backward-guided policy viable under compute-matched decoding by ensuring the additional cost scales with the active candidate set rather than the full sequence length.

%% file: sec/experiments.tex
\begin{figure*}[t]
    \centering
    \includegraphics[width=\linewidth]{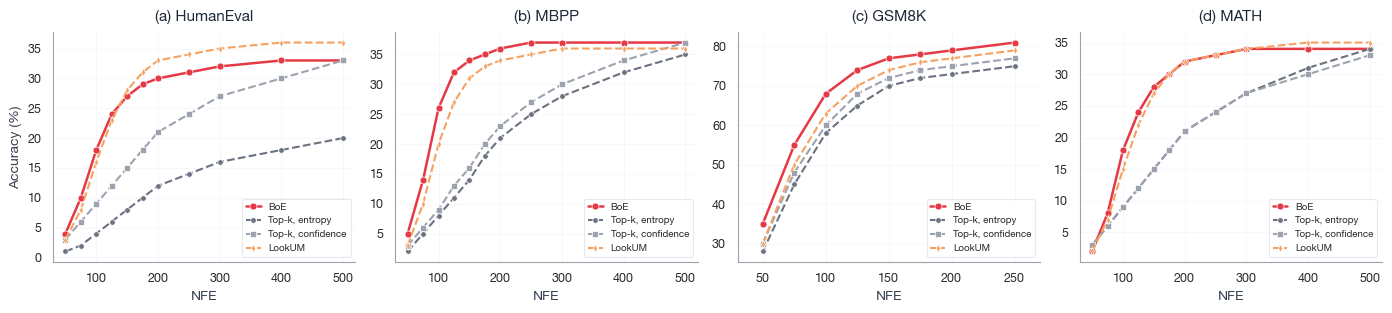}
    \caption{Pass@1 accuracy vs.\ full \texttt{max\_gen\_len} NFE on LLaDA-8B.
        We report compute-matched scaling curves on HumanEval, MBPP, GSM8K, and MATH500 for four training-free samplers: Top-$k$ (entropy), Top-$k$ (confidence), LookUM, and BoE (ours).
        Following ~\citep{ebsampler} evaluation protocol, we fix \texttt{max\_gen\_len} and measure NFE as the number of denoiser evaluations required to unmask all tokens in the generation region.
        }
    \label{fig:placeholder}
\end{figure*}

\begin{table*}[!htp]
\centering
\small
\setlength{\tabcolsep}{8.3pt}
\caption{
LLaDA-8B benchmark performance (pass@1). All methods are training-free.
Best result per column is shown in \best{bold}, second-best in \second{underlined}.
}
\begin{tabular}{l
c >{\columncolor{ColGray}}c
c >{\columncolor{ColGray}}c
c >{\columncolor{ColGray}}c
c >{\columncolor{ColGray}}c
c >{\columncolor{ColGray}}c
c >{\columncolor{ColGray}}c
}
\toprule
\multirow{2}{*}{\textbf{Method}} &
\multicolumn{2}{c}{\textbf{MBPP}} &
\multicolumn{2}{c}{\textbf{HumanEval}} &
\multicolumn{2}{c}{\textbf{GSM8K}} &
\multicolumn{2}{c}{\textbf{MATH500}} &
\multicolumn{2}{c}{\textbf{Countdown}} &
\multicolumn{2}{c}{\textbf{Sudoku}} \\
& 128 & 256 & 128 & 256 & 128 & 256 & 128 & 256 & 128 & 256 & 128 & 256 \\
\midrule
Confidence
& 28.6 & 28.4
& 19.5 & 32.0
& 68.3 & 76.7
& 26.0 & 32.4
& 20.3 & \best{21.9}
& 1.4 & 27.4 \\

Margin
& 28.6 & 28.4
& 25.6 & 31.0
& 67.1 & 76.1
& 28.4 & 34.4
& 19.1 & 20.7
& 21.8 & 27.8 \\

Entropy
& 27.2 & 24.4
& 19.5 & 15.9
& 66.7 & 75.4
& 26.0 & 33.0
& 21.9 & 20.3
& 0.0 & 12.0 \\

PC-Sampler
& 24.0 & 25.2
& 13.4 & 30.5
& 67.3 & 73.7
& 25.2 & 32.4
& \best{26.5} & 20.3
& 23.2 & 24.0 \\

ReMDM
& 28.6 & 28.4
& 17.1 & 29.3
& 69.1 & 77.9
& 27.4 & 33.0
& 25.3 & 17.2
& 0.4 & 22.8 \\

LookUM
& \second{30.5} & \second{36.2}
& \best{27.4} & \best{35.9}
& \second{72.7} & \second{79.3}
& \second{28.8} & \best{34.6}
& 25.4 & \best{23.1}
& \second{25.0} & \second{28.0} \\

\midrule
\rowcolor{OursGreen}
\textbf{BoE (ours)}
& \best{31.6} & \best{37.4}
& \second{26.9} & \second{33.1}
& \best{73.9} & \best{80.6}
& \best{29.6} & \second{34.2}
& \second{24.8} & {21.4}
& \best{25.7} & \best{29.2} \\
\bottomrule
\end{tabular}
\label{tab:main_results_llada}
\end{table*}

\begin{table*}[t]
\centering
\small
\setlength{\tabcolsep}{8.3pt}
\caption{LLaDA-1.5 benchmark performance (pass@1).
All methods are training-free and evaluated under compute-matched decoding. Best and second-best results per column are highlighted.
}
\begin{tabular}{l
c >{\columncolor{ColGray}}c
c >{\columncolor{ColGray}}c
c >{\columncolor{ColGray}}c
c >{\columncolor{ColGray}}c
c >{\columncolor{ColGray}}c
c >{\columncolor{ColGray}}c
}
\toprule
\multirow{2}{*}{\textbf{Method}} &
\multicolumn{2}{c}{\textbf{MBPP}} &
\multicolumn{2}{c}{\textbf{HumanEval}} &
\multicolumn{2}{c}{\textbf{GSM8K}} &
\multicolumn{2}{c}{\textbf{MATH500}} &
\multicolumn{2}{c}{\textbf{Countdown}} &
\multicolumn{2}{c}{\textbf{Sudoku}} \\
& 128 & 256 & 128 & 256 & 128 & 256 & 128 & 256 & 128 & 256 & 128 & 256 \\
\midrule
Confidence
& 40.3 & 38.6
& 26.8 & 23.7
& 69.5 & 79.4
& 28.6 & 32.6
& 20.3 & 23.4
& 1.4 & 27.4 \\

Margin
& 39.5 & 38.1
& \best{31.1} & 27.4
& 71.3 & 78.3
& 27.2 & 35.0
& 24.6 & 14.0
& 21.8 & 27.8 \\

Entropy
& 40.7 & 36.5
& 20.7 & 23.2
& 69.7 & 77.0
& 28.2 & 32.2
& 23.0 & 12.9
& 0.0 & 12.0 \\

PC-Sampler
& 42.8 & 39.6
& 23.8 & 23.8
& 70.1 & 77.3
& 26.6 & 32.2
& 25.4 & 19.1
& 25.6 & 27.2 \\

ReMDM
& 41.9 & 39.3
& 28.1 & \second{29.8}
& 70.4 & 80.1
& 27.4 & 34.0
& 23.4 & 19.9
& 0.4 & 22.8 \\

LookUM
& \second{45.0} & \second{43.6}
& \second{30.7} & \best{33.5}
& \second{74.5} & \second{82.3}
& \second{29.2} & \second{35.8}
& \best{27.3} & 17.9
& \second{26.8} & \second{28.0} \\

\midrule
\rowcolor{OursGreen}
\textbf{BoE (ours)}
& \best{45.4} & \best{44.1}
& {28.3} & {27.6}
& \best{74.9} & \best{83.1}
& \best{30.4} & \best{36.9}
& \second{26.4} & \second{21.1}
& \best{27.9} & \best{29.4} \\
\bottomrule
\end{tabular}
\label{tab:main_results_llada15}
\end{table*}

\section{Experiments}
\label{sec:experiments}

We evaluate BoE steering on (i) standard reasoning and code generation benchmarks, (ii) distribution-level text generation quality, and (iii) structured logic puzzles. Following recent inference-time samplers for masked diffusion models~\citep{ebsampler,remdm,lookum,klass}, our evaluation emphasizes compute-matched comparisons, accuracy--compute Pareto curves, and explicit overhead accounting for any additional inference-time computation.

\subsection{Experimental Setup}
\label{sec:exp_setup}

\textbf{Backbone models.}
Unless stated otherwise, we report main results on \texttt{LLaDA-8B-Base} \citep{llada}, and the RL-tuned \texttt{LLaDA-1.5} model when available \citep{llada1.5}. We use sequence lengths $L\in\{128,256\}$ and the standard decoding setup used by prior sampler, two tokens per denoising step under a fixed horizon $T$ unless otherwise noted \citep{lookum,pathp2}. All experiments were run on NVIDIA H200 GPUs. 

\textbf{Datasets and tasks.}
We evaluate on diverse benchmarks spanning mathematics, code synthesis, and planning/constraint solving:
(i) \textbf{GSM8K}~\citep{gsm8k} and \textbf{MATH500}~\citep{math500} for math reasoning,
(ii) \textbf{HumanEval}~\citep{humaneval} and \textbf{MBPP}~\citep{mbpp} for code generation,
and (iii) \textbf{Sudoku} and \textbf{Countdown} for structured reasoning/planning, following the protocols in recent MDM sampling work \citep{lookum,pathp2}. For distribution-level quality, we evaluate unconditional/continuation generation on \textbf{OpenWebText} (OWT) using \textbf{MAUVE}~\citep{mauve}, consistent with ReMDM \citep{remdm}.



\textbf{Generation length control.}
To avoid inflated compute from producing unused tokens, we adopt a task-specific \texttt{generate\_until} logic analogous to autoregressive evaluation, extended to MDMs by requiring that all tokens up to the stop pattern are unmasked \citep{ebsampler}. 

\textbf{Baselines.}
We compare BoE to strong inference-time samplers representing the main families of MDM decoding. Local-uncertainty greedy schedules: Confidence~\citep{maskgit}, Margin~\citep{kim2025train}, and Negative Entropy \citep{koh2024plm} unmasking. PC-Sampler: calibrated confidence with position- and frequency-aware adjustments~\citep{pcsampler}. Entropy-bounded scheduling: adaptive unmasking based on uncertainty constraints \citep{ebsampler}. ReMDM~\citep{remdm}, PC Sampler~\citep{pcsampler} and LookUM~\citep{lookum}. We report \textbf{BoE} (full) and \textbf{BoE + \texttt{ActiveQueryAttention}}. The latter keeps the forward pass unchanged but sparsifies the backward computation to an active candidate set, enabling compute-matched deployment at practical sequence lengths (Section~\ref{sec:aqa}).

Unless specified otherwise, $\mathcal{C}_t$ is the top-$r$ masked positions by confidence, with $r$ set as a fraction of the remaining masked tokens, $r = \rho |\mathcal{M}_t|$, $\rho \in \{0.10, 0.25, 0.50\}$. For \texttt{ActiveQueryAttention}, active set $\mathcal{A}_t := \mathcal{C}_t$. Value of entropy floor $h_t = h_{\max}(1 - t/T)$ with $h_{\max}$ tuned on a small validation subset.

\subsection{Main Results}
\label{sec:exp_main}

\textbf{Benchmark performance.}
Table~\ref{tab:main_results_llada} shows that BoE consistently improves pass@1 over strong training-free baselines on \texttt{LLaDA-8B-Base} under compute-matched decoding at $L\in\{128,256\}$.
Across the four core code+math benchmarks that we sweep for scaling curves (HumanEval, MBPP, GSM8K, MATH500), BoE attains the best result at both context lengths on MBPP and GSM8K, and is either best or second-best on HumanEval and MATH500.
Notably, BoE improves over LookUM at both lengths on MBPP (+1.1/+1.2 points at $L=128/256$) and GSM8K (+1.2/+1.3), and remains competitive on HumanEval (within 0.5/2.8 points) and MATH500 (best at $L=128$, within 0.4 at $L=256$).
On structured tasks, BoE also improves solve rate on Sudoku (25.7/29.2 vs.\ 25.0/28.0 for LookUM), while remaining competitive on Countdown (24.8/21.4).

Table~\ref{tab:main_results_llada15} confirms that these gains persist on the stronger RL-tuned \texttt{LLaDA-1.5} checkpoint.
BoE achieves the best results on MBPP, GSM8K, MATH500, and Sudoku at both lengths, and is competitive on HumanEval.
These results support our central claim that prioritizing unmasking actions by predicted future entropy reduction improves trajectory quality beyond local uncertainty heuristics.

\textbf{Accuracy--compute Pareto.}
Single operating points can hide meaningful trade-offs across compute budgets.
Following the EB-Sampler protocol~\citep{ebsampler}, Figure~\ref{fig:placeholder} reports compute-matched scaling curves on \texttt{LLaDA-8B} for HumanEval, MBPP, GSM8K, and MATH500 using four training-free samplers (Top-$k$ Entropy, Top-$k$ Confidence, LookUM, and BoE).
Across datasets, BoE yields a uniformly better (or equal) accuracy--NFE frontier relative to local Top-$k$ baselines, and improves over LookUM at matched \texttt{max\_gen\_len} NFE on the majority of operating points, especially in math reasoning where early pivot resolution is critical.

\textbf{Distribution-level text quality.}
Beyond exact-match benchmarks, we evaluate distributional fidelity on OpenWebText using MAUVE~\citep{mauve}.
At $T{=}128$, Table~\ref{tab:mauve_owt_128} shows that MDLM+BoE achieves the best MAUVE while preserving strong diversity (entropy) and improving GenPPL relative to FB/ReMDM.
This indicates that BoE's future-uncertainty-sensitive scheduling can improve global coherence and distributional similarity, not only task accuracy.

\begin{table*}[t]
\centering
\small

\begin{minipage}[t]{0.52\linewidth}
\centering
\caption{Quality vs.\ inference-time compute on OpenWebText at $T{=}128$.
We evaluate the trade-off between sample quality, likelihood, and diversity under compute-matched decoding. We report MAUVE(~\citep{mauve}) as a distributional similarity metric against the reference corpus, generative perplexity (Gen PPL) computed from model likelihoods, and predictive entropy as a measure of output diversity. 
}
\vspace{2pt}

\setlength{\tabcolsep}{3pt}
\begin{tabular}{lccc}
\toprule
\textbf{Method} &
\textbf{MAUVE} $\uparrow$ &
\textbf{Gen PPL} $\downarrow$ &
\textbf{Entropy} $\uparrow$ \\
\midrule
SEDD (absorb)
& 0.007 & 119.2 & \best{5.65} \\

MDLM (locked-in)
& 0.015 & 61.5 & \second{5.52} \\

MDLM + FB
& \second{0.064} & 42.8 & 5.44 \\

MDLM + DFM
& 0.041 & \best{37.9} & 5.31 \\

ReMDM (re-mask)
& 0.057 & \second{42.5} & 5.43 \\

\midrule
\rowcolor{OursGreen}
\textbf{MDLM + BoE}
& \best{0.069} & 47.6 & 5.48 \\
\bottomrule
\end{tabular}
\label{tab:mauve_owt_128}
\end{minipage}
\hfill
\begin{minipage}[t]{0.45\linewidth}
\centering
\caption{Accuracy and runtime on GSM8K under compute-matched decoding ($L{=}128$).
We report pass@1 accuracy and end-to-end evaluation runtime on \textbf{8$\times$H200 GPUs}.
All methods use one-token-per-step sampling.
}
\setlength{\tabcolsep}{1pt}
\begin{tabular}{lcc}
\toprule
\textbf{Method / Variant} &
\textbf{GSM8K (pass@1)} &
\textbf{Runtime (hrs.)} \\
\midrule
Confidence (Top-1)
& 68.3
& 1.76 \\

\midrule
\rowcolor{OursGreen}
\textbf{BoE (full)}
& \best{73.9}
& \best{1.94} \\

\quad w/o confidence gating
& 72.8
& 1.92 \\

\quad w/o anti-collapse
& 72.4
& 1.92 \\

\quad w/o ActiveQueryAttention
& 73.6
& 2.71 \\

\quad active fraction $\rho{=}0.10$
& 71.9
& 1.86 \\

\quad active fraction $\rho{=}0.25$
& 73.7
& 1.90 \\

\quad active fraction $\rho{=}0.50$
& 73.9
& 2.03 \\
\bottomrule
\end{tabular}
\label{tab:runtime_gsm8k_h200}
\end{minipage}

\end{table*}

\subsection{Ablations}
\label{sec:exp_ablation}

We isolate which BoE components drive accuracy gains versus which are required for practical efficiency, using GSM8K at $L{=}128$ under compute-matched decoding on 8$\times$H200 GPUs (Table~\ref{tab:runtime_gsm8k_h200}).

\textbf{Lookahead gradient signal is the primary driver.}
BoE (full) achieves $73.9$ pass@1, indicating that the backward-on-entropy scoring yields a clear gain over local schedules under the same decoding setup.
Removing confidence gating or anti-collapse reduces accuracy (to $72.8$ and $72.4$), consistent with these components stabilizing the backward signal and preventing premature overcommitment, but the remaining performance stays close to full BoE, suggesting the core benefit comes from the future-entropy gradient criterion itself.

\textbf{\texttt{ActiveQueryAttention} is necessary for practical efficiency.}
Disabling \texttt{ActiveQueryAttention} (dense backward) preserves accuracy ($73.6$) but substantially increases runtime (2.71 hrs vs.\ 1.94 hrs), confirming that sparse backward is the key systems mechanism that makes BoE compute-feasible at long contexts.
In contrast, the full method achieves the top accuracy while remaining in the same runtime regime as baselines.

\textbf{Candidate/active-set fraction provides a tunable speed--accuracy knob.}
Varying the active fraction $\rho$ shows a smooth trade-off.
A smaller active set ($\rho{=}0.10$) reduces runtime (1.86 hrs) but incurs an accuracy drop (71.9), while $\rho{=}0.25$ largely recovers accuracy (73.7) with near-minimal runtime (1.90).
Using a larger active set ($\rho{=}0.50$) matches full accuracy (73.9) at higher runtime (2.03).
Overall, $\rho{=}0.25$ is a strong default that preserves most of the gain while keeping backward overhead controlled.

\textbf{Takeaway.}
These ablations support two conclusions:
(i) BoE's one-step future-entropy objective is responsible for the majority of accuracy improvements, and
(ii) \texttt{ActiveQueryAttention} is essential to realize these gains at practical wall-clock cost by restricting the backward computation to a small candidate set.











%% file: sec/conclusion.tex
\section{Conclusion}
\label{sec:conclusion}

Masked diffusion language models enable parallel, revisable decoding, yet most existing samplers schedule unmasking using local uncertainty heuristics and rely on stochastic repair to recover from early commitments. We proposed \textbf{Backward-on-Entropy (BoE) Steering}, a training-free, model-agnostic framework that treats unmasking as trajectory optimization and prioritizes updates by their predicted reduction in future masked uncertainty. BoE instantiates this objective via a single backward signal that yields a per-position Token Importance Score (TIS) from a first-order embedding-space surrogate, and we introduce \texttt{ActiveQueryAttention} to make backward-guided selection practical by restricting gradients to an active query set. Across compute-matched evaluations, BoE is designed to improve the quality--compute tradeoff over greedy local-uncertainty schedules, remasking-based samplers, and forward-lookahead policies while preserving practical latency and memory. More broadly, BoE provides a principled interface between diffusion decoding and control, suggesting a reusable direction for efficient inference in discrete diffusion and other masked generative models.

%% file: sec/appendix.tex
\clearpage
\appendix
\onecolumn

\tableofcontents

\newpage

\section{Appendix}

\subsection{Additional Notation and Setup}
\label{sec:appendix_setup}

We follow the notation of Section~\ref{sec:method}.
At reverse step $t$, the partially masked sequence is $x_t\in(\mathcal{V}\cup\{m\})^{L}$ with masked index set
$\mathcal{M}_t := \{i \in [L] : x_t^i = m\}$.
The denoiser induces per-position predictive distributions
$p^i_\theta(\cdot\mid x_t,t)\in\Delta^{|\mathcal{V}|-1}$ for $i\in\mathcal{M}_t$.
The masked predictive entropy objective is
\begin{equation}
\mathcal{H}_t(x_t)
=
\sum_{i\in\mathcal{M}_t}
H\!\left(p^i_\theta(\cdot\mid x_t,t)\right).
\end{equation}
BoE forms a candidate set $\mathcal{C}_t\subseteq\mathcal{M}_t$ and computes a Token Importance Score (TIS)
using the relaxed surrogate next-step entropy $\widetilde{\mathcal{H}}_{t-1}(\mathcal{U})$ defined in
Eq.~\eqref{eq:surrogate_next_entropy_tminus}.

Let $E\in\mathbb{R}^{|\mathcal{V}|\times d}$ denote the token embedding matrix, and $e_m\in\mathbb{R}^{d}$ the mask embedding.
For a masked position $i\in\mathcal{M}_t$, define
\begin{equation}
\pi_t^i := p^i_\theta(\cdot\mid x_t,t),\qquad
\tilde e_t^i := E^\top \pi_t^i,\qquad
\Delta e_t^i := \tilde e_t^i - e_m,
\end{equation}
as in Eq.~\eqref{eq:soft_embedding_t}.
Unless otherwise stated, all gradients in this appendix are with respect to the injected embedding at a candidate position,
and all forward computations remain identical to the baseline denoiser.

\subsection{Proofs for Section~\ref{sec:method}}
\label{sec:appendix_theory}

\subsubsection{Proof of Theorem~\ref{thm:first_order_tminus}}
\label{sec:appendix_proof_thm_firstorder}

We restate the interpolation and surrogate objective for clarity.
Fix step $t$ and a masked position $i\in\mathcal{M}_t$.
Define the single-position interpolation (Eq.~\eqref{eq:interp_tminus}):
\begin{equation}
e^i(\alpha) = e_m + \alpha\,\Delta e_t^i,\qquad \alpha\in[0,1],
\end{equation}
with all other masked positions kept as $m$.
Let $\widetilde{\mathcal{H}}_{t-1}(\alpha;i)$ denote the surrogate next-step masked entropy
(Eq.~\eqref{eq:surrogate_next_entropy_tminus}) computed by injecting $e^i(\alpha)$ at position $i$ and leaving all other positions unchanged.

\textbf{Claim (Theorem~\ref{thm:first_order_tminus}).}
Assume $\widetilde{\mathcal{H}}_{t-1}(\alpha;i)$ is twice continuously differentiable on $[0,1]$ and
\begin{equation}
\sup_{\alpha\in[0,1]}
\left|
\frac{d^2}{d\alpha^2}\widetilde{\mathcal{H}}_{t-1}(\alpha;i)
\right|
\le M.
\label{eq:app_curv_bound}
\end{equation}
Define the relaxed one-step entropy reduction
$\Delta\widetilde{\mathcal{H}}_{t-1}(i):=\widetilde{\mathcal{H}}_{t-1}(0;i)-\widetilde{\mathcal{H}}_{t-1}(1;i)$.
Then
\begin{equation}
\Delta\widetilde{\mathcal{H}}_{t-1}(i)
=
\mathrm{TIS}_t^i + \varepsilon_t^i,
\qquad
|\varepsilon_t^i|\le \frac{M}{2},
\end{equation}
where $\mathrm{TIS}_t^i := -\langle g_t^i,\Delta e_t^i\rangle$ and $g_t^i := \nabla_{e_i}\widetilde{\mathcal{H}}_{t-1}(0;i)$.

\begin{proof}
Define the scalar function
$f(\alpha) := \widetilde{\mathcal{H}}_{t-1}(\alpha;i)$.
By Taylor's theorem with remainder around $\alpha=0$, for some $\xi\in(0,1)$,
\begin{equation}
f(1) = f(0) + f'(0)\cdot 1 + \frac{1}{2} f''(\xi)\cdot 1^2.
\label{eq:app_taylor_scalar}
\end{equation}
Rearranging yields
\begin{equation}
f(0) - f(1) = - f'(0) - \frac{1}{2}f''(\xi).
\label{eq:app_rearrange}
\end{equation}
By the chain rule, because $e^i(\alpha)=e_m+\alpha\Delta e_t^i$,
\begin{equation}
f'(0)
=
\left\langle
\nabla_{e_i}\widetilde{\mathcal{H}}_{t-1}(0;i),
\frac{d}{d\alpha} e^i(\alpha)\big|_{\alpha=0}
\right\rangle
=
\left\langle g_t^i,\Delta e_t^i\right\rangle.
\end{equation}
Thus $-f'(0) = -\langle g_t^i,\Delta e_t^i\rangle = \mathrm{TIS}_t^i$.
Finally, using the curvature bound \eqref{eq:app_curv_bound} in \eqref{eq:app_rearrange} gives
\[
\varepsilon_t^i := -\frac{1}{2}f''(\xi),
\qquad
|\varepsilon_t^i|
\le \frac{1}{2}\sup_{\alpha\in[0,1]}|f''(\alpha)|
\le \frac{M}{2}.
\]
This proves the theorem.
\end{proof}

\subsubsection{Proof of Corollary~\ref{cor:ordering}}
\label{sec:appendix_cor_ordering}

\begin{proof}
By Theorem~\ref{thm:first_order_tminus}, for $i$ and $j$ we have
\[
\Delta\widetilde{\mathcal{H}}_{t-1}(i) = \mathrm{TIS}_t^i + \varepsilon_t^i,
\qquad
\Delta\widetilde{\mathcal{H}}_{t-1}(j) = \mathrm{TIS}_t^j + \varepsilon_t^j,
\]
with $|\varepsilon_t^i|\le M/2$ and $|\varepsilon_t^j|\le M/2$.
Therefore
\[
\Delta\widetilde{\mathcal{H}}_{t-1}(i) - \Delta\widetilde{\mathcal{H}}_{t-1}(j)
=
(\mathrm{TIS}_t^i - \mathrm{TIS}_t^j) + (\varepsilon_t^i - \varepsilon_t^j).
\]
The worst case is $\varepsilon_t^i-\varepsilon_t^j \ge -M$.
If $\mathrm{TIS}_t^i - \mathrm{TIS}_t^j > M$, then the right-hand side is strictly positive, implying
$\Delta\widetilde{\mathcal{H}}_{t-1}(i) > \Delta\widetilde{\mathcal{H}}_{t-1}(j)$.
\end{proof}

\subsubsection{Boundedness of confidence-gated scores}
\label{sec:appendix_gating}

This proposition complements Eq.~\eqref{eq:confidence_gating_tminus} by showing that gating preserves a uniform bound.

\begin{proposition}[Boundedness of gated TIS]
\label{prop:app_bounded_tis}
Assume $\|\Delta e_t^i\|_2 \le D$ and $\|g_t^i\|_2 \le G$ for all $i,t$.
Then for any confidence gate $c_t^i\in[0,1]$,
\begin{equation}
\left|\widehat{\mathrm{TIS}}_t^i\right|
=
\left|c_t^i\,\mathrm{TIS}_t^i\right|
\le GD.
\end{equation}
\end{proposition}

\begin{proof}
By Cauchy--Schwarz,
$|\mathrm{TIS}_t^i| = |\langle g_t^i,\Delta e_t^i\rangle|
\le \|g_t^i\|_2\|\Delta e_t^i\|_2 \le GD$.
Multiplying by $c_t^i\in[0,1]$ preserves the inequality.
\end{proof}

\subsection{Implementation Details}
\label{sec:appendix_code}


\subsubsection{Surrogate state construction for Eq.~\eqref{eq:surrogate_next_entropy_tminus}}
\label{sec:app_surrogate_impl}

At each reverse step $t$, we run:
\begin{enumerate}
\item \textbf{Forward at $t$:} compute $\pi_t^i = p_\theta^i(\cdot\mid x_t,t)$ for $i\in\mathcal{M}_t$.
\item \textbf{Candidate injection:} for $i\in\mathcal{C}_t$, compute $\tilde e_t^i=E^\top \pi_t^i$ and inject it after the embedding layer.
\end{enumerate}
We use a stop-gradient through $\pi_t^i$ when forming $\tilde e_t^i$:
\begin{equation}
\tilde e_t^i \leftarrow \mathrm{stopgrad}(E^\top \pi_t^i),
\label{eq:app_stopgrad_softwrite}
\end{equation}
so that the backward signal reflects the sensitivity of the next-step surrogate uncertainty to \emph{revealing} information at candidate positions,
rather than coupling the score to how $\pi_t^i$ itself changes under infinitesimal perturbations of $x_t$.

\textbf{Time index consistency.}
Our surrogate uses a single additional denoiser call at $(t-1)$ as written in Eq.~\eqref{eq:surrogate_next_entropy_tminus}.
That is, $\widetilde{\mathcal{H}}_{t-1}$ is computed by running the denoiser at time $(t-1)$ on the relaxed input $\tilde x_{t-1}(\mathcal{U})$.

\subsubsection{ActiveQueryAttention:}
\label{sec:app_aqa_impl}

ActiveQueryAttention (Section~\ref{sec:aqa}) preserves the forward attention outputs and therefore preserves the forward logits exactly.
It only restricts gradients to a designated active query set $\mathcal{A}_t$ (we use $\mathcal{A}_t=\mathcal{C}_t$).

Let $Y\in\mathbb{R}^{L\times d}$ be the output of a self-attention block (post-attention projection).
Given an indicator $m\in\{0,1\}^L$ with $m_i=1$ iff $i\in\mathcal{A}_t$, implement
\begin{equation}
\widetilde Y_i := m_i\,Y_i + (1-m_i)\,\mathrm{stopgrad}(Y_i).
\label{eq:app_aqa_detach}
\end{equation}
The model consumes $\widetilde Y$ in place of $Y$.
This leaves the forward numerically unchanged while preventing backward flow through inactive query rows.

\subsubsection{Complexity statement for Eq.~\eqref{eq:aqa_complexity}}
\label{sec:app_aqa_complexity}

Under standard attention, the dominant backward term scales as $O(L^2 d)$ due to gradients for all $L$ queries against all $L$ keys/values.
Under ActiveQueryAttention, gradients are only required for $|\mathcal{A}_t|$ query rows against all $L$ keys/values, yielding $O(|\mathcal{A}_t|Ld)$.
This matches Eq.~\eqref{eq:aqa_complexity} in the main text.

\subsubsection{Compute accounting and NFE}
\label{sec:app_nfe}

We count \textbf{every denoiser invocation} as one function evaluation (NFE), including:
(i) the standard forward at $t$,
(ii) the surrogate forward at $(t-1)$ used to compute $\widetilde{\mathcal{H}}_{t-1}$,
and (iii) any additional calls introduced by baselines (e.g., remasking transitions).
Wall-clock latency is reported under identical precision and batch settings across methods.

\subsection{Related Work}
\label{sec:related}

We position BoE within: (i) autoregressive LLMs and inference acceleration, (ii) discrete diffusion and masked diffusion LLMs, (iii) inference-time unmasking schedules (myopic, entropy-bounded, and learned), (iv) correction and inference-time scaling via remasking or MCMC-style moves, and (v) lookahead/control-theoretic formulations for masked diffusion decoding. BoE is \emph{training-free} and \emph{model-agnostic}, and differs from prior lookahead policies by using a \emph{single backward pass} to approximate \emph{future-entropy reduction} at the token level, made practical via \texttt{ActiveQueryAttention}.

\subsubsection{Autoregressive LLMs and Inference Acceleration}
Autoregressive (AR) transformers remain the dominant paradigm for large-scale language modeling and instruction following \citep{brown2020language,raffel2020exploring,chowdhery2023palm,llama,llama2,team2024gemma,yang2025qwen3,achiam2023gpt}.
A key advantage of AR decoding is its compatibility with KV caching, which amortizes attention computation over the prefix and enables numerous inference optimizations and cache-management strategies \citep{kv-survey,zeng2024context}.
However, AR decoding is inherently sequential and cannot revise earlier tokens without explicit refinement loops.
Recent work also highlights systematic limitations tied to directional supervision and sequential factorization (e.g., the reversal curse) \citep{berglund2023reversal,golovneva2024reverse}, motivating non-AR alternatives and refinement-based inference.
BoE targets masked diffusion LLMs specifically: rather than managing KV cache (which is not directly applicable in the same form), BoE improves the \emph{token-reveal policy} that governs parallelizable refinement trajectories.

\subsubsection{Discrete Diffusion and Masked Diffusion LLMs}
Discrete diffusion in categorical spaces has been developed via masking-based or transition-kernel formulations \citep{austin2021structured,maskgit}.
For language, masked diffusion language modeling has recently matured into competitive LLM-scale systems, including MDLM \citep{mdlm}, LLaDA \citep{llada}, and diffusion LLM variants such as Dream \citep{ye2025dream}, as well as post-training refinements such as LLaDA-1.5 \citep{llada1.5}.
These models typically define a corruption (masking) process and learn denoising conditionals $p_\theta(x_{t-1}^i\mid x_t,t)$ for masked positions, leaving substantial degrees of freedom at inference time in \emph{which} positions are updated and \emph{how} compute is allocated across steps.
BoE is an inference-time method that leaves training unchanged and acts only on the decoding policy.

\subsubsection{Unmasking Schedules: Myopic, Entropy-Bounded, and Learned Policies}
A large fraction of masked diffusion decoding quality is governed by the unmasking schedule.
Classical heuristics select positions using \emph{current-step} uncertainty proxies (confidence, margin, entropy), often stabilizing generation but remaining myopic in long-horizon reasoning.
Entropy-bounded unmasking (EB-Sampler) explicitly controls the uncertainty profile and adaptively allocates budget to satisfy an entropy constraint, yielding a strong compute--quality trade-off \citep{ebsampler}.
Several contemporary directions study token ordering and inference policies more broadly, including entropy-adaptive Gibbs-style procedures for PLM-based discrete diffusion \citep{koh2024plm}, analyses of ordering effects \citep{kim2025train}, and approaches that \emph{learn} unmasking policies \citep{jazbec2025learning}.
KLASS proposes KL-guided fast inference in masked diffusion models, using divergence-based criteria to accelerate decoding \citep{klass}.
BoE differs from these lines in criterion and mechanism: it scores actions by approximating their \emph{effect on next-step masked uncertainty} (Eq.~\eqref{eq:boe_objective_tminus}--\eqref{eq:surrogate_next_entropy_tminus}), rather than selecting low-entropy tokens (myopic) or enforcing a target entropy band (EB-style). Practically, BoE computes token-level utility via the directional derivative proxy (Theorem~\ref{thm:first_order_tminus}) using a single surrogate forward and a sparse backward pass.

\subsubsection{Correction Mechanisms and Inference-Time Scaling}
A core failure mode in iterative refinement is \emph{error lock-in}: an early incorrect commitment can constrain later steps.
ReMDM introduces inference-time scaling for discrete diffusion by \emph{remasking} and revisiting tokens to enable correction, improving robustness at the cost of additional transitions and potentially higher effective NFE \citep{remdm}.
Related inference-time scaling strategies also include more global sampling moves, e.g., particle-Gibbs style methods for diffusion language models \citep{gibbs-particle}.
BoE is complementary: by prioritizing \emph{load-bearing pivots} earlier, BoE aims to prevent collapse trajectories under compute-matched decoding; remasking or MCMC moves can be layered on top as an optional repair mechanism.

\subsubsection{Lookahead and Control-Theoretic Views of Masked Diffusion Decoding}
Lookahead policies attempt to account for downstream effects of token reveals.
LookUM evaluates downstream consequences using additional forward computations, improving over purely myopic schedules but incurring multi-evaluation overhead per decision \citep{lookum}.
Other recent work formulates diffusion sampling as planning or control, including path planning for masked diffusion sampling \citep{pathp2} and stochastic optimal control formulations \citep{zhu2025mdns}, as well as anchored/structured variants \citep{adlm}.
BoE is most directly aligned with the \emph{control} viewpoint: it instantiates an explicit one-step lookahead objective in terms of future masked entropy (Eq.~\eqref{eq:boe_objective_tminus}), but makes it practical by (i) a continuous relaxation for candidate writes (Eq.~\eqref{eq:soft_embedding_t}--\eqref{eq:surrogate_next_entropy_tminus}) and (ii) a \emph{single-backward} approximation (Eq.~\eqref{eq:tis_def_tminus}, Theorem~\ref{thm:first_order_tminus}), rather than enumerating forward rollouts.

\subsubsection{Gradient-Based and Information-Theoretic Selection Principles}
Information gain and entropy reduction are classical principles in active learning and experimental design; in generation, they motivate selecting actions that reduce posterior uncertainty.
In discrete diffusion decoding, directly optimizing such criteria is challenging due to the combinatorial action space.
BoE imports an information-gain-like objective into masked diffusion decoding by treating unmasking as an information acquisition action: the Token Importance Score approximates the marginal value of revealing a position via a first-order entropy-reduction surrogate with a quantifiable remainder (Theorem~\ref{thm:first_order_tminus}).
A practical barrier is the cost of backpropagating through attention for all queries; BoE resolves this via \texttt{ActiveQueryAttention} (Section~\ref{sec:aqa}), which preserves exact forward logits while restricting backward to the active candidate set.


\subsection{Limitations, Broader Impacts, and Future Work}
\label{sec:limitations}

\subsubsection{Limitations}

BoE requires one surrogate forward pass and one backward pass per decoding iteration. While ActiveQueryAttention reduces the backward complexity from $O(L^2)$ to $O(|\mathcal{A}_t|L)$, the constant factors and engineering effort remain nontrivial. For small models or short sequences, the gains from improved steering may not offset the added cost. The Token Importance Score relies on a first-order Taylor approximation in embedding space (Theorem~\ref{thm:first_order_tminus}) and a soft embedding relaxation of discrete tokens. In regimes with highly non-smooth decision boundaries, such as code generation or symbolic reasoning with brittle constraints, this approximation may be less predictive of true downstream entropy reduction. BoE is most effective when the candidate set $\mathcal{C}_t$ is substantially smaller than the masked set $\mathcal{M}_t$. Overly aggressive prefiltering can exclude structurally critical positions, degrading performance. Although we provide ablations over candidate fractions, robust automatic candidate selection remains an open problem. BoE explicitly optimizes future entropy reduction, which correlates with uncertainty but not always with correctness. Models can be confidently wrong, and entropy reduction may occur for spurious explanations. The anti-collapse regularizer mitigates early overcommitment, but additional calibration or verifier-based signals may be necessary in high-stakes settings.

\subsubsection{Broader Impacts}

BoE provides a principled, training-free mechanism for improving inference quality and efficiency in masked diffusion models. By reducing error lock-in and improving reasoning reliability, it may enable lower-latency deployment of diffusion-based LLMs in practical applications. More broadly, the control-theoretic framing may inform future work on structured decoding beyond heuristic sampling. As with any improvement to generative inference, BoE can lower the cost of producing large volumes of high-quality text, potentially amplifying misuse such as spam or misinformation. BoE does not introduce new capabilities, but standard safeguards remain necessary. Because BoE explicitly encourages entropy reduction, there is a risk of increased overconfidence without commensurate gains in correctness in some regimes. Careful evaluation on calibration, robustness, and hallucination benchmarks is required before deployment in high-stakes domains.

\subsubsection{Future Work}

BoE currently performs a one-step entropy lookahead. Extending the framework to limited-horizon planning (e.g., two to three steps) using low-rank or factorized approximations could capture longer-range dependencies without prohibitive cost. Entropy is an intrinsic objective. Incorporating lightweight verifiers, such as arithmetic checks or constraint satisfaction proxies, could align unmasking decisions with correctness rather than uncertainty alone. For multimodal MDMs, future work should explore modality-weighted entropy objectives and constraints that balance text and image token selection, preventing domination by high-cardinality modalities. Automated selection of $\mathcal{C}_t$ using offline calibration, bandit-style heuristics, or structural priors could improve robustness while preserving the training-free nature of BoE.